\documentclass[twoside]{article}

\usepackage{algorithm}
\usepackage{amsmath}

\usepackage[utf8]{inputenc} 
\usepackage[T1]{fontenc}    
\usepackage{url}            
\usepackage{booktabs}       
\usepackage{amsfonts}       
\usepackage{amsthm}         
\usepackage{nicefrac}       
\usepackage{microtype}      
\usepackage{xcolor}         

\usepackage{amsmath}

\usepackage{amssymb}
\usepackage{mathtools}
\usepackage{bm}
\usepackage{soul}
\usepackage{color}
\usepackage{float}
\usepackage[bottom]{footmisc}

\usepackage[authoryear,round]{natbib}
\setcitestyle{round,aysep={,},citesep={;}}
\usepackage[colorlinks=true,linkcolor=black,citecolor=blue,urlcolor=blue]{hyperref}

\usepackage{algorithm}
\usepackage{algorithmicx}
\usepackage{algpseudocode}

\usepackage{subcaption}
\usepackage{graphicx}

\usepackage{caption}

\theoremstyle{plain}
\newtheorem{theorem}{Theorem}[section]

\theoremstyle{definition}

\theoremstyle{remark}
\newtheorem{remark}{Remark}

\usepackage[textsize=tiny]{todonotes}



\usepackage{siunitx}

\newcolumntype{P}{r@{\,$\pm$\,}l}
\newcommand{\val}[2]{#1 & #2}
\newcommand{\best}[2]{\textbf{#1} & \textbf{#2}}

\usepackage[accepted]{aistats2026}
\begin{document}

%

%

\twocolumn[

\aistatstitle{Counterfactual Credit Guided Bayesian Optimization}

\aistatsauthor{ Qiyu Wei$^{1}$ \And Haowei Wang$^{2}$ \And  Richard Allmendinger$^{3}$ \And Mauricio A. Álvarez$^{1}$}

\aistatsaddress{ $^{1}$Department of Computer Science, The University of Manchester  \\ $^{2}$Department of ISEM, National University of Singapore  \\
$^{3}$Alliance Manchester Business School, University of Manchester} 
]

\begin{abstract}
Bayesian optimization has emerged as a prominent methodology for optimizing expensive black-box functions by leveraging Gaussian process surrogates, which focus on capturing the global characteristics of the objective function. However, in numerous practical scenarios, the primary objective is not to construct an exhaustive global surrogate, but rather to quickly pinpoint the global optimum. 
Due to the aleatoric nature of the sequential optimization problem and its dependence on the quality of the surrogate model and the initial design, it is restrictive to assume that all observed samples contribute equally to the discovery of the optimum in this context. 
In this paper, we introduce Counterfactual Credit Guided Bayesian Optimization (CCGBO), a novel framework that explicitly quantifies the contribution of individual historical observations through counterfactual credit. By incorporating counterfactual credit into the acquisition function, our approach can selectively allocate resources in areas where optimal solutions are most likely to occur. We prove that CCGBO retains sublinear regret. Empirical evaluations on various synthetic and real-world benchmarks demonstrate that CCGBO consistently reduces simple regret and accelerates convergence to the global optimum.
\end{abstract}

\section{Introduction}

Bayesian Optimization (BO) is a powerful framework for optimizing expensive-to-evaluate black-box functions, demonstrating impressive efficiency across machine learning for hyperparameter tuning and experimental design tasks~\citep{frazier2018tutorial,wang2023recent}. Central to BO is the balance between exploration and exploitation~\citep{garnett2023bayesian}. 
Traditionally, BO algorithms achieve this balance implicitly via acquisition functions~\citep{archetti2019acquisition,wilson2018maximizing,wang2025convergence} that guide sampling decisions based on the posterior uncertainty and predicted outcomes from Gaussian Process (GP) models~\citep{mackay1998introduction}. 

However, in BO the standard exploration–exploitation trade‑off has limitations: it often wastes evaluations in suboptimal regions. BO is fundamentally a resource allocation problem. 
When budgets are extremely tight, finding the promising region early is more valuable than making marginal improvements later, relying solely on exploration and exploitation can be inefficient. 
Meanwhile, a limitation of conventional BO methods is the implicit assumption that all historical observations contribute equally to the optimization progress. 
In real-world landscapes, some specific samples provide more informative evidence of the global optimum than others~\citep{koh2017understanding}, yet this heterogeneity is ignored and we may fail to identify key information. As a result, budget allocation is often inefficiently wasted in suboptimal areas, resulting in the budget not being reasonably allocated to areas that are more likely to have optimal values and slowing down convergence.

\begin{figure*}[!t]
    \centering
    \begin{minipage}[b]{\textwidth}
        \centering
        \includegraphics[width=1\textwidth]{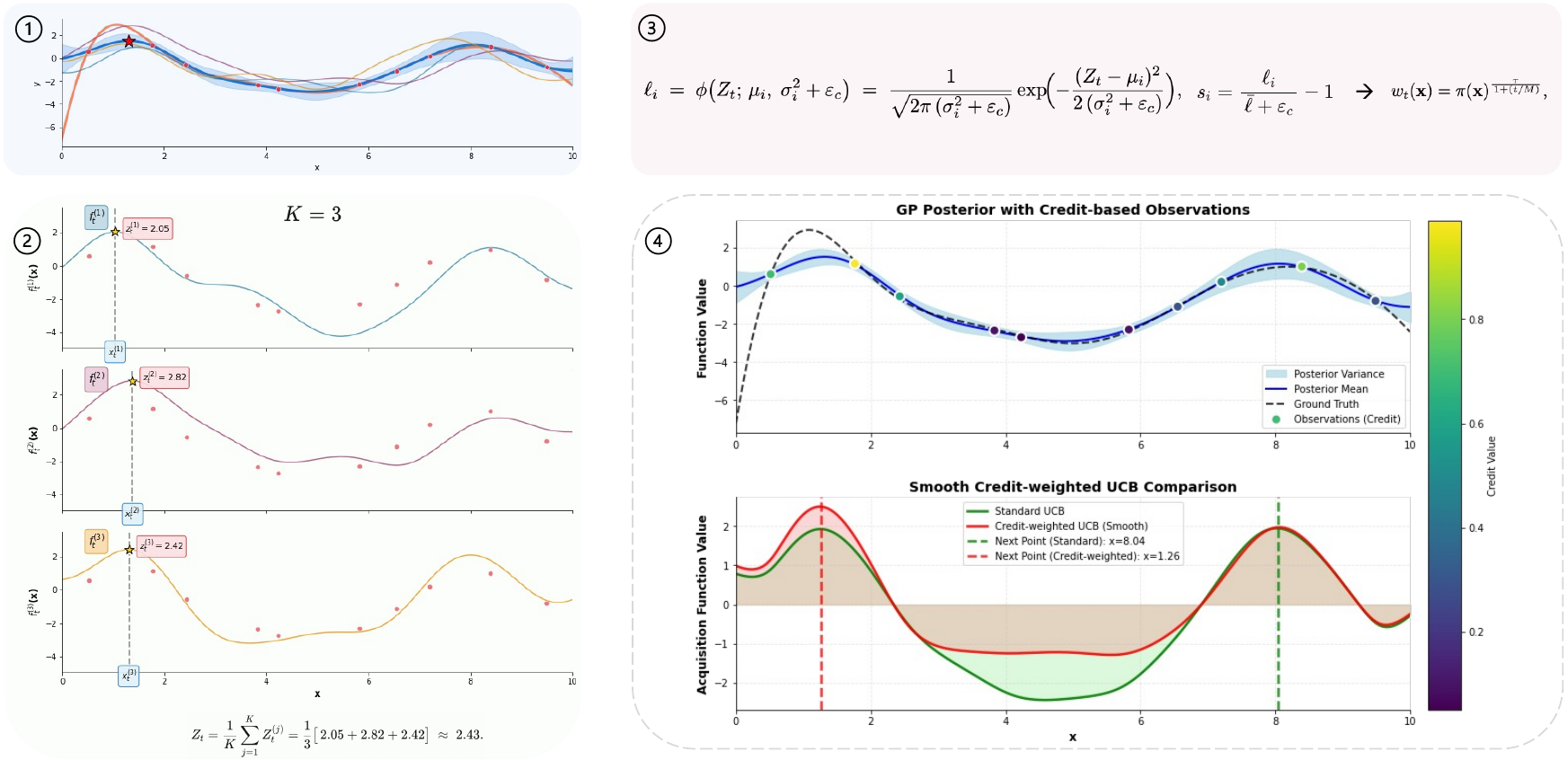}
    \end{minipage}
    \caption{Illustration of Credit-Weighted UCB. Our goal is to maximize the objective function.
(1) From the GP posterior at iteration $t$, we draw several paths; existing observations are shown as dots and the true optimum is marked by $\star$.
(2) In this illustration, for $K=3$ posterior samples, we compute each sample’s maximizer $x_t^{(j)}$ and maximum $Z_t^{(j)}$, and form the Monte Carlo proxy of the global maximum $Z_t$.
(3) For every observed point $x_i$, we assign a counterfactual credit $\ell_i$, normalize and propagate it, then obtain the weight $w_t(x)$.
(4) The right-hand color bar shows the counterfactual credit. By incorporating counterfactual credit into the UCB (red), our method concentrates exploitation on high‑contribution regions, yielding a next query that targets the true optimum compared to the standard UCB (green).
}

    \label{fig:Illustration} 
\end{figure*}

Some existing methods add regional constraints to avoid less informative or unsafe samples and focus on promising areas~\citep{sui2015safe, acerbi2017practical, eriksson2019scalable, kim2020safe}. However, these methods often rely on manual thresholds and do not have adaptive capabilities, and using local optimization instead of global optimization may lead to suboptimal results. 
In addition, some work uses expert knowledge to ``Bring Human Back Into Loop''~\citep{hvarfner2022pi,hvarfner2023general}, but this often requires good priors. Without accurate priors, performance can be worse than standard BO, and in real applications, we cannot have a good prior for all settings. 
However, these methods essentially rely on either external thresholds or expert priors. Therefore, they are unable to explicitly characterize which region is more promising from the observations itself.

To explicitly address these limitations, we propose Counterfactual Credit Guided Bayesian Optimization (CCGBO), which introduces an explicit mechanism to quantify the contribution of each historical observation using counterfactual reasoning~\citep{harutyunyan2019hindsight, mesnard2020counterfactual, meulemans2023would}. 
Because credits are derived directly from the GP posterior, CCGBO does not require external priors.
Specifically, counterfactual credit works as a proxy of contribution, assessing how significantly the prediction of the current optimum would degrade if a particular observation were absent, directly answering the question: ``How significantly would our prediction of the current optimum deteriorate if a particular observation were absent?'' By explicitly quantifying these contributions, 
CCGBO moves beyond traditional exploration-exploitation trade-off to a richer three-dimensional trade-off: exploration-exploitation-importance.
Integrating counterfactual credit into the BO framework empowers BO to allocate sampling resources more efficiently and target areas that truly matter, enabling precise identification and prioritization of regions critically influencing rapid convergence to the optimum. Figure~\ref{fig:Illustration} visually illustrates how Credit-Weighted UCB effectively directs exploration toward areas of high contribution, demonstrating more targeted queries than standard UCB \citep{srinivas2009gaussian}.  
In summary, our contributions are threefold:
\begin{itemize}
\item \textbf{Counterfactual credit.}  We introduce counterfactual credit as an efficient proxy for the per‑sample contribution score of BO samples, which does not require manual specification. Quantifying and exploiting the varying counterfactual contribution credit of historical samples, enabling a trade‑off on exploration-exploitation-importance;
\item \textbf{Theoretical analysis.}  We establish that the optimum proxy tracks the true optimum, which justifies CCGBO's early focusing on high-value regions. We further prove that counterfactual credit‑weighted UCB retains sublinear regret, and show that it inherits the theoretical properties of the GP-UCB acquisition function; 
\item \textbf{Empirical validations.}  We develop a counterfactual credit‑weighted acquisition function, offering a modular toolkit that retains compatibility with any GP‑BO backbone. We empirically validate the performance improvements provided by our counterfactual credit-guided approach across diverse optimization scenarios.
\end{itemize}

\section{Related Works}
\paragraph{Knowledge from Previous Experiments.}
In recent years, there has been growing interest in leveraging data from previous experiments to improve the efficiency of BO.  
\emph{(1) OutlierBO.}
\cite{martinez2018practical} addresses the issue that spurious observations can corrupt the GP surrogate and mislead acquisition. They propose a two‐stage method using a Student-\emph{t} likelihood in the GP to identify and remove outliers. 
\emph{(2) Continuous optimization in non-stationary environments.}
When the objective drifts over time, stale data may lose relevance. \cite{bogunovic2016time} extend GP-UCB to time-varying settings by periodically restarting the model to discard old data or gradually forgetting past observations via exponentially decaying weights. Building on this,~\cite{deng2022weighted} introduces an explicit weight sequence $\omega_t$ into both the GP inference and acquisition computation, making recent points carry greater influence.  
\emph{(3) Historical data utilization in multi-task and transfer learning.}
In contexts where data from related tasks are available, \cite{swersky2013multi} propose a multi‐task GP that learns a task covariance kernel, thereby implicitly weighting historical data according to inter‐task similarity. \cite{feurer2018scalable} fit a separate GP to each previous optimization run and combine their predictions with weights based on rank similarity between the source and target tasks. \cite{hakhamaneshi2021jumbo} develop a multi-task BO, which trains a warm GP on large offline datasets and a cold GP on online samples, and dynamically trades off their suggestions based on task relevance. More recently, \cite{mahboubi2025point} introduced a point‐by‐point transfer strategy that evaluates the value of each candidate historical sample via a Gaussian mixture model and incorporates only the observations most aligned with the target task.
Overall, these methods rely on pre-specified mechanisms. We instead compute per–sample weights using a counterfactual method.

\paragraph{Embedding Priors over Function Optimum.}
Embedding informative priors over the location of a function's global optimum within BO has attracted growing interest.
\emph{(1) Expert‐Driven Priors.}
A number of works have sought to inject user beliefs directly into the acquisition process. \cite{ramachandran2020incorporating} warp the input space according to a prior distribution over promising regions. \cite{souza2021bayesian} and \cite{hvarfner2022pi} generalize this idea, defining a prior over the optimum that modulates any standard acquisition function. \cite{huang2022bayesian} complement these by actively querying experts for pairwise preferences and encoding the resulting belief model into the GP prior. \cite{wu2017bayesian} exploit the available gradient information as a form of expert knowledge, leading to the derivative‐inclusive knowledge gradient algorithm that can drastically speed up convergence.
\emph{(2) Physics‐Informed and Simulator‐Based Priors.}
In domains where there are partial physical models, integrating these laws into BO can produce substantial gains. \cite{khatamsaz2023physics}  embed known material‐science relationships directly into GP kernels to optimize NiTi alloy processing, while \cite{kobayashi2025physics} infuse epitaxial growth laws into the prior and demonstrate accurate extrapolation for III–V semiconductor deposition. \cite{boltz2025leveraging} further show that using fast accelerator‐simulation models as a GP mean function permits online tuning of particle accelerators with fewer expensive measurements.
\emph{(3) Functional‐Structure Priors.}
Beyond beliefs and data, exploiting the known structural properties of the objective can also guide BO. \cite{li2017bayesian} inject monotonicity constraints via virtual derivative observations to enforce that the surrogate respects known increasing or decreasing trends. \cite{kandasamy2015high} address high‐dimensional spaces by assuming an additive decomposition of the objective. However, these methods all depend heavily on having correctly specified priors, whether from experts, simulators, or structural assumptions, and can perform poorly or lose robustness when such prior information is unavailable, biased, or misspecified.

\section{Preliminaries}

In this section, we summarize BO setup and notation, and introduce an auxiliary prior augmented acquisition.

\subsection{Bayesian Optimization}
BO is a sequential decision-making framework that aims to find the global optimum of an expensive black-box function. BO employs a surrogate model, typically a GP, to approximate the unknown function, and an acquisition function to balance exploration and exploitation when selecting the next evaluation point.

\textbf{Problem Setting.}  
We consider the problem of maximizing an expensive black-box objective function
\(
  f: \mathcal{X}\to\mathbb{R},~
  \mathbf{x}^* = \arg\max_{\mathbf{x}\in\mathcal{X}} f(\mathbf{x}),
\)
where \(\mathcal{X}\subset\mathbb{R}^d\) is a compact domain and each evaluation of \(f(\mathbf{x})\) incurs high cost. At iteration \(t\), we have collected observations
\(
  \mathcal{D}_t = \bigl\{(\mathbf{x}_i,y_i)\bigr\}_{i=1}^t,~
  y_i = f(\mathbf{x}_i) + \varepsilon_i,~
  \varepsilon_i\sim\mathcal{N}(0,\sigma_n^2),
\)
where \(\sigma_n^2\) denotes the variance of observation noise. BO maintains a GP surrogate that, conditioned on \(\mathcal{D}_t\), yields a posterior mean \(\mu_t(\mathbf{x})\) and standard deviation \(\sigma_t(\mathbf{x})\). An acquisition function \(\alpha(\mathbf{x})\), built from \(\mu_t\) and \(\sigma_t\), then prescribes the next query point:
\(
  \mathbf{x}_{t+1} = \arg\max_{\mathbf{x}\in\mathcal{X}}\,\alpha(\mathbf{x}),
\)
after which we observe
\(
  y_{t+1} = f(\mathbf{x}_{t+1}) + \varepsilon_{t+1},
\)
and update
\(
  \mathcal{D}_{t+1} = \mathcal{D}_t \cup \{(\mathbf{x}_{t+1},y_{t+1})\}.
\)
This process repeats until a stopping criterion is met, and we return
\(
  {\mathbf{x}}^* 
\) as our estimated optimizer.

\textbf{Surrogate Model.}  
We place a GP prior on \(f\):
\(
f(\mathbf{x})\sim\mathcal{GP}\bigl(m(\mathbf{x}),\,k(\mathbf{x},\mathbf{x}')\bigr),
\)
where \(m(\mathbf{x})\) is the prior mean and \(k(\mathbf{x},\mathbf{x}')\) is a positive-definite covariance kernel.  Conditioned on \(\mathcal{D}_t\), the posterior at a test point \(\mathbf{x}\) is
\(
  \mu_t(\mathbf{x}) = m(\mathbf{x}) + k\bigl(\mathbf{x},\mathbf{X}_{1:t}\bigr)\bigl[\mathbf{K}_{1:t} + \sigma_n^2\mathbf{I}\bigr]^{-1}\bigl(\mathbf{y}_{1:t}-m(\mathbf{X}_{1:t})\bigr),
\)
\(
  \sigma_t^2(\mathbf{x}) = k(\mathbf{x},\mathbf{x}) - k\bigl(\mathbf{x},\mathbf{X}_{1:t}\bigr)\bigl[\mathbf{K}_{1:t} + \sigma_n^2\mathbf{I}\bigr]^{-1}k\bigl(\mathbf{X}_{1:t},\mathbf{x}\bigr),
\)
where \(\mathbf{X}_{1:t}=[\mathbf{x}_1,\dots,\mathbf{x}_t]\), \(\mathbf{y}_{1:t}=[y_1,\dots,y_t]^\top\), and \(\mathbf{K}_{1:t}\) is the kernel matrix with entries \(K_{ij}=k(\mathbf{x}_i,\mathbf{x}_j)\).  

\textbf{Acquisition Function.}  
Given the posterior \(\mathcal{N}\bigl(\mu_t(\mathbf{x}),\sigma_t^2(\mathbf{x})\bigr)\), BO selects the next point by maximizing an acquisition function.  For the Upper Confidence Bound (UCB), we use
\(
  \mathrm{UCB}_t(\mathbf{x}) = \mu_t(\mathbf{x}) + \beta_t \,\sigma_t(\mathbf{x}),
  ~ \beta_t>0,
\)
and choose
\(
  \mathbf{x}_{t+1} = \arg\max_{\mathbf{x}\in\mathcal{X}} \mathrm{UCB}_t(\mathbf{x}).
\)
The new observation \((\mathbf{x}_{t+1},y_{t+1})\) is added to \(\mathcal{D}_{t+1}\), and the procedure repeats until termination. Upon completion, the best observed point \({\mathbf{x}}^*\) is returned.

\subsection{Augmented Acquisition Functions}
In many real-world tasks, one also has access to auxiliary information \(a\in\mathcal{A}\) (e.g.\ domain clues, contextual features) that correlates with the optimizer's location. We encode this via a prior
\begin{equation}
  \kappa(\mathbf{x}\mid a)
  = \Pr\bigl(\mathbf{x}=\arg\max_{\mathbf{x}'}f(\mathbf{x}') \,\bigm|\, a\bigr),
\end{equation}
which places high mass on \(\mathbf{x}\) likely to be optimal given \(a\).  We then augment any base acquisition \(\alpha(\mathbf{x})\) by
\begin{equation}
  \alpha_a\bigl(\mathbf{x},a\bigr)
  = \alpha \bigl(\mathbf{x}\bigr)\;\kappa(\mathbf{x}\mid a),
  ~
  \mathbf{x}_{t+1}=\arg\max_{\mathbf{x}} \alpha_a(\mathbf{x},a).
\end{equation}
By forming \(\alpha_a(\mathbf{x})\) as the product of the GP-based acquisition \(\alpha(\mathbf{x})\) and the prior \(\kappa( \mathbf{x} \mid a)\), we focus sampling on points where both the surrogate model predicts high utility and the auxiliary information indicates a high likelihood of optimality.

\section{Counterfactual Credit Guided Bayesian Optimization (CCGBO)}

In this section, we introduce CCGBO, we first motivate and define counterfactual credit, then derive a posterior-based estimator and its propagation to candidates, and finally integrate it into a credit-weighted acquisition.

\subsection{Counterfactual Credit for BO}

In BO, each observed sample contributes unequally to the discovery of the global optimum.  At each iteration \(t\), selecting a candidate point \(\mathbf{x}_t\) not only yields an immediate function evaluation \(y_t\), but also influences the future trajectory of the optimization and, consequently, the speed with which the true optimum is located. Our goal is to assign a credit \(c(\mathbf{x}_t)\) to each observed sample that quantifies its contribution to the downstream search for the optimum.
A natural ``Monte Carlo'' style approach is to simulate multiple complete BO trajectories \(\{\tau_i\}\) starting from the current dataset and to estimate
\(
  \bar R(\mathbf{x}_t)
  \;=\;
  \frac{1}{K}
  \sum_{i=1}^K
    R\bigl(\tau_i \mid \mathbf{x}_t \bigr),
\)
where \(K\) is the number of simulated trajectories, and \(R(\tau_i\mid \mathbf{x}_t)\) denotes the total discounted reward (e.g., reduction in best seen value) along trajectory \(\tau_i\) given that sample \(\mathbf{x}_t\) was selected at iteration \(t\). However, because each simulated trajectory is subject to environmental noise and the stochasticity of the acquisition policy, the variance of this estimator is high. 
This forward and rollout-based sequential simulation leads to low sample efficiency in the optimization process, which is not suitable for real-time BO.

Ideally, one would like to perform counterfactual updates for all candidate samples in a single trajectory: Even if a particular point \(\mathbf{x}'\) was not chosen at iteration \(t\), we would still like to estimate ``what would have happened'' had we observed \(f(\mathbf{x}')\). Existing BO methods do not provide this capability, since they only update the posterior at the actually sampled points.
Therefore, we propose a counterfactual credit computation as a proxy of contribution that (i)~leverages the GP posterior to estimate, for each candidate \(\mathbf{x}\), the hypothetical contribution to finding the optimum, and (ii)~can be evaluated in closed form and with low computational overhead at each iteration. Here we remodel the calculation method of contribution from \emph{``how does choosing an observation $\mathbf{x}$ in the past affect the finding of optimal''} to \emph{``given the optimal, how relevant was the choice of observation $\mathbf{x}$ in the past to achieve it?''}. 
If a sample contributes a lot about the optimal location, it usually means that it is located in or close to the area in space where high values are most likely to appear, and it will have a larger credit accordingly. Although observations with poor results will also contribute to finding the optimal value, their contribution is smaller than that of observations near the optimal value, which will be reflected in the credit.
By incorporating counterfactual credit into the acquisition function, we move beyond the exploration-exploitation trade-off to a richer framework that uses the heterogeneous informativeness of samples. This enables BO to allocate sampling resources more efficiently, concentrating effort on regions that are most relevant to locating the global optimum.

\subsection{Computation of Counterfactual Credit}

At iteration \(t\), let the dataset be
\(
\mathcal{D}_t = \{(\mathbf{x}_i, y_i)\}_{i=1}^t,
\)
and let the GP posterior at any \(\mathbf{x}\in\mathcal{X}\) have mean \(\mu_t(\mathbf{x})\) and standard deviation \(\sigma_t(\mathbf{x})\).  Rather than directly using the observed maximum \(\max_i y_i\), we construct a current global optimum proxy \(Z_t\) via a Monte Carlo approximation to the UCB criterion:
Draw \(K\) independent sample paths from the GP posterior
\(
f_t^{(j)}(\cdot)\sim\mathcal{GP}\bigl(\mu_t(\cdot),\,k_t(\cdot,\cdot)\bigr),
~ j=1,\dots,K.
\)
For each sample \(j\), we approximate its maximizer
\(
\mathbf{x}_t^{(j)} 
\;=\; 
\arg\max_{\mathbf{x}\in\mathcal{X}} f_t^{(j)}(\mathbf{x}),
\)
and record the corresponding current global optimal proxy
\(
Z_t^{(j)} 
\;=\; 
f_t^{(j)}\bigl(\mathbf{x}_t^{(j)}\bigr).
\)
We then aggregate the \(K\) sample maxima into a single proxy  
\(
Z_t 
\;=\;
\frac{1}{K}\sum_{j=1}^K Z_t^{(j)}.
\)
This Monte Carlo proxy captures both the posterior mean variance trade-off and uncertainty about the true global maximum.
Denote for brevity at each observed location \(\mathbf{x}_i\):
\(
\mu_i = \mu_t(\mathbf{x}_i),
~
\sigma_i = \sigma_t(\mathbf{x}_i).
\)
We compute a likelihood score quantifying how likely each \(\mathbf{x}_i\) is to have produced \(Z_t\):
\begin{equation}
\begin{split}
\ell_i
&\;=\;
\phi\bigl(Z_t;\,\mu_i,\;\sigma_i^2 + \varepsilon_c\bigr)\\
&\;=\;
\frac{1}{\sqrt{2\pi\,(\sigma_i^2 + \varepsilon_c)}}
\exp\!\Bigl(-\frac{(Z_t - \mu_i)^2}{2\,(\sigma_i^2 + \varepsilon_c)}\Bigr),
\end{split}
\end{equation}
where \(\phi(\cdot;m,v)\) is the Gaussian density with mean \(m\), variance \(v\), and \(\varepsilon_c>0\) is a small constant used to prevent division by zero and stabilize the denominator.
Next, we introduce the unconditional baseline
\(
\bar\ell
\;=\;
\frac{1}{t}\sum_{j=1}^t \ell_j.
\)
And define the raw counterfactual score
\(
s_i
\;=\;
\frac{\ell_i}{\,\bar\ell + \varepsilon_c\,}
\;-\;
1,
\)
so that \(s_i>0\) indicates \(\mathbf{x}_i\) positively contributes to achieving the global proxy \(Z_t\).
Finally, we normalize and bound these scores into credits \(c_i\in[r_{\min},\,r_{\max}]\).  Let \(t\) be the current number of observations.  We first calculate a normalized rank
\(
r_i
\;=\;
\frac{\bigl|\{\,j:1\le j\le t,\;s_j \le s_i\}\bigr|-1}{\,t-1\,}
\;\in\;[0,1].
\)
Then we linearly map this rank into the credit interval:
\(
c_i
\;=\;
r_{\min}
\;+\;
(r_{\max}-r_{\min})\,r_i,
\)
where \(r_{\min}=0.1\) and \(r_{\max}=1\). We set \(r_{\min}=0.1\) to avoid collapsing any weight to zero, which would prematurely exclude regions.
The resulting \(\{c_i\}\) form our counterfactual credits, which can be seamlessly incorporated into augmented acquisition functions, thus realizing a three-way exploration-exploitation-importance strategy in BO. Next, we will explain how to extend the discrete credits \(c_i\) to arbitrary candidate points.

\subsection{Counterfactual Credit-Weighted Acquisition Function}

\begin{algorithm}[t]
\caption{Counterfactual Credit Guided BO}
\label{alg:ccgbo_compact}
\begin{algorithmic}[1]
\State \textbf{Input:} objective \(f\), domain \(\mathcal X\), initial data $\mathcal{D}_t$, total iterations \(N\), credit weight \(\lambda\), UCB factor \(\{\beta_t\}\), decay \(\tau\), proxy samples \(K\).
\State \textbf{Output:} estimated maximizer \(\mathbf{x}^*\).
\For{\(t=1,\dots,N\)}
  \State  Train on \(  \mathcal{D}_t \to \mu_t,\,\sigma_t\).
    \For{\(j=1,\dots,K\)}
      \State Sample \(f_t^{(j)}\sim\mathcal{GP}(\mu_t,k_t)\).
      \State \(\mathbf{x}_t^{(j)}\gets\arg\max_{\mathbf{x}}f_t^{(j)}(\mathbf{x})\);
        \(Z_t^{(j)}\gets f_t^{(j)}(\mathbf{x}_t^{(j)})\).
    \EndFor
  \State Obtain optimistic estimate \(Z_t\gets\frac1K\sum_jZ_t^{(j)}\).
  \State Compute per‐point credits \(\{c_i\}\), credit field \(\pi(\mathbf{x})\) and weighting factor \(w_t(\mathbf{x})\).
  \State Compute UCB \(\alpha(\mathbf{x}) = \mu(\mathbf{x})+\beta_t\,\sigma(\mathbf{x})\); shift: \(\tilde{\alpha}(\mathbf{x}) = \alpha(\mathbf{x}) - \min_{\mathbf{x}'} \alpha(\mathbf{x}')\).
  \State Obtain Credit-Weighted UCB \(\alpha^{\mathrm{ccg}}(\mathbf{x}) = \bigl[(1-\lambda)+\lambda\,w_t(\mathbf{x})\bigr]\,\tilde{\alpha}(\mathbf{x})\).
  \State \(\mathbf{x}_{\rm new}\gets\arg\max_{\mathbf{x}\in X_{\rm cand}}\alpha^{\mathrm{ccg}}(\mathbf{x})\), \(y_{\rm new}\gets f(\mathbf{x}_{\rm new})\).
  \State \(\mathcal{D}_{t+1} \leftarrow \mathcal{D}_t \cup \{(\mathbf{x}_{\mathrm{new}},\,y_{\mathrm{new}})\}\).
\EndFor
\State \Return \(\mathbf{x}^*=\arg\max_{(\mathbf{x},y)\in\mathcal{D}_t}y\).
\end{algorithmic}
\end{algorithm}

In standard BO, acquisition functions such as UCB balance exploration and exploitation based on the posterior predictive mean \(\mu(\mathbf{x})\) and standard deviation \(\sigma(\mathbf{x})\):
\(
  \operatorname{UCB}(\mathbf{x}) = \mu(\mathbf{x}) + \beta_t\, \sigma(\mathbf{x}),
\)
where \(\beta_t > 0\) is a tunable parameter. However, this criterion treats all past observations as equally informative, ignoring their varying contributions.
To address this, we introduce a credit-weighted acquisition function that integrates counterfactual credits into an acquisition function.  Let
\(
  \{(\mathbf{x}_i, y_i)\}_{i=1}^t
  ~\text{and}~
  \{c_i\}_{i=1}^t
\)
denote the evaluated points and their credits. We proceed in two steps:

(1)\emph{ Propagating Credits to Continuous Candidates.}
Given \(X_{\rm cand}\), we estimate a candidate credit \(c(\mathbf{x})\) by averaging the credits of its \(H\) nearest neighbors in the evaluated set:
\(
  c(\mathbf{x})
  = \frac{1}{H}\sum_{j \in \mathcal{N}_H(\mathbf{x})} c_j,
\)
where \(\mathcal{N}_H(\mathbf{x})\) are the indices of the \(H\) nearest \(\mathbf{x}_j\) to \(\mathbf{x}\) under Euclidean distance. This simple KNN‐based propagation yields a smooth credit field
\(\pi(\mathbf{x}) = \tfrac{c(\mathbf{x})}{\max_j c_j}\)
over \(\mathcal{X}\).

(2)\emph{ Credit-Weighted UCB.}
To ensure that the acquisition values remain non-negative before applying credit weights, we first shift the UCB by its minimum over candidates. Let \(\alpha(\mathbf{x}) = \mu(\mathbf{x}) + \beta_t\,\sigma(\mathbf{x})\) denote the standard UCB. We define the shifted acquisition as
\(
  \tilde{\alpha}(\mathbf{x}) = \alpha(\mathbf{x}) - \min_{\mathbf{x}' \in X_{\mathrm{cand}}} \alpha(\mathbf{x}'),
\)
so that \(\tilde{\alpha}(\mathbf{x}) \ge 0\) for all candidates. Since the shift is an \(\mathbf{x}\)-independent constant, it does not change the argmax of the unweighted acquisition. We then define the credit-weighted acquisition function:
\begin{equation}
  \mathrm{UCB}_{\rm credit}(\mathbf{x})
  = \bigl[(1-\lambda) + \lambda\,w_t(\mathbf{x})\bigr]\,\tilde{\alpha}(\mathbf{x}),
\end{equation}
where \(\lambda \in [0,1]\) modulates the strength of the credit influence, and the weighting factor \(w_t(\mathbf{x})\) is computed based on the counterfactual evaluation of each observation's contribution:
\begin{equation}
w_t(\mathbf{x}) = \pi(\mathbf{x})^{\frac{\tau}{\,1 + (t/M)}}
\end{equation}
is a per‐iteration weight.  Here \(\pi(\mathbf x)\) is the normalized counterfactual credit for point \(\mathbf x\), \(\tau>0\) controls the sensitivity to credit differences, \(t\) is the current iteration index, and the constant \(M\) sets the ``half‐life'' of credit influence.  At \(t=0\), \(w_0(\mathbf x)=\pi(\mathbf x)^\tau\) strongly biases early sampling toward high‐credit regions; by \(t=M\), the exponent halves to \(\tau/2\), reducing extra weighting.  As \(t\) grows further, the denominator dominates and \(w_t(\mathbf x)\to1\), smoothly reverting the acquisition back to the standard UCB form \(\mu+\beta_t\,\sigma\).
In practice, $M$ and $\tau$ govern the decay of the credit effect, we therefore set them accordingly to match different evaluation budgets.

\paragraph{On the Role of Acquisition Shift.}
A notable property of credit-weighted acquisition is that, unlike standard GP-UCB which is invariant to additive shifts, our acquisition function is sensitive to the shift value because multiplying by non-uniform weights \(w_t(\mathbf{x})\) breaks this invariance.
Subtracting the minimum UCB ensures that the lowest-UCB candidate has acquisition value zero, while higher-UCB candidates retain positive surplus values. This design amplifies the discriminative effect of credit weights: candidates in high-credit regions with large UCB surplus are strongly favored, whereas low-UCB candidates are suppressed regardless of their credit values. Without this shift, negative UCB values multiplied by small weights could yield higher acquisition values than positive UCB values multiplied by larger weights, leading to counterintuitive sampling decisions.

\section{Theoretical Analysis}

Here, we first establish that our Monte Carlo proxy of the current optimum tracks the true optimum. Let \(Z_t\) denote the average optimum of \(K\) posterior sample paths drawn at iteration \(t\). Theoretical analysis shows that \(Z_t\) concentrates around the true optimal value \(f(\mathbf{x}^\star)\) with high probability, with a decomposition into a model bias term and an MC error term. We now state the assumptions and theorem.

\textbf{Assumptions 1}
There exists $\beta_t>0$ such that the GP-UCB confidence event
\(
\mathcal E_t:=\Big\{\,\lvert f(\mathbf{x})-\mu_t(\mathbf{x})\rvert \le \beta_t\,\sigma_t(\mathbf{x})\ \ \forall \mathbf{x}\in\mathcal X\,\Big\}
\)
holds with probability at least $1-\delta$~\citep{srinivas2009gaussian}. 
And there exist constants $c_0,c_1>0$ such that for the zero-mean posterior process 
$\varepsilon_t:=f_t-\mu_t$ we have
$\mathbb{E}\!\big[\sup_{\mathbf{x}\in\mathcal X}\varepsilon_t(\mathbf{x})\,\bigm|\,\mathcal D_t\big]\le c_0 S_t$~\citep{adler2007random},
and, conditionally on $\mathcal D_t$, the centered variables $Z_t^{(j)}-\mathbb{E}[Z_t^{(j)}\mid\mathcal D_t]$
are $(c_1 S_t)$-sub-Gaussian ($\psi_2$-norms are $\le c_1 S_t$)~\citep{ledoux2013probability}.

\begin{theorem}\label{prop:mc-consistency}
Let $\mathcal X\subset\mathbb{R}^d$ be compact, and at iteration $t$ let the GP posterior
be $f_t\sim\mathcal{GP}(\mu_t,k_t)$ with $\sigma_t^2(\mathbf{x})=k_t(\mathbf{x},\mathbf{x})$ and
$S_t=\sup_{\mathbf{x}\in\mathcal X}\sigma_t(\mathbf{x})<\infty$.
Draw $K$ i.i.d.\ posterior sample paths $\{f_t^{(j)}\}_{j=1}^K$, and define
$Z_t^{(j)}=\max_{\mathbf{x}\in\mathcal X} f_t^{(j)}(\mathbf{x})$ and $Z_t=\tfrac1K\sum_{j=1}^K Z_t^{(j)}$.
Let $\mathbf{x}^\star\in\arg\max_{\mathbf{x}\in\mathcal X} f(\mathbf{x})$.
For any $\delta'\in(0,1)$, there exists a constant $C>0$ (depending only on $c_1$) such that
\begin{equation}\label{eq:mc-consistency-bound}
\Pr\!  \left( 
  \bigl|Z_t - f(\mathbf{x}^\star)\bigr| 
  \le 
  \beta_t S_t
  +
  C\,S_t\,\sqrt{\tfrac{\log(1/\delta')}{K}}
\right) \ge 1-\delta-\delta'.
\end{equation}
\end{theorem}

\begin{remark}
By Theorem~\ref{prop:mc-consistency}, with high probability we have
\(
|Z_t - f(\mathbf{x}^\star)| \le (\beta_t + c_0) S_t + C S_t \sqrt{\log(1/\delta')/K},
\)
so the MC proxy \(Z_t\) is close to the true optimum \(f(\mathbf{x}^\star)\).
Consequently, for observed points \(\mathbf{x}_i\) that lie in or near the optimal region, the quantity
\(
(Z_t - \mu_t(\mathbf{x}_i))^2
\)
tends to be smaller, which makes the likelihood score
\(
\ell_i \;=\; \phi\!\big(Z_t;\, \mu_t(\mathbf{x}_i),\, \sigma_t^2(\mathbf{x}_i)+\varepsilon_c\big)
\)
larger. Then resulting credit field \(\pi(\mathbf{x})\) attains higher values around the optimum, making
\(
w_t(\mathbf{x})=\pi(\mathbf{x})^{\tau/(1+t/M)}
\)
also larger. Hence the credit-weighted acquisition
\(
\alpha_t(\mathbf{x})=\big[(1-\lambda)+\lambda\,w_t(\mathbf{x})\big]\big[\mu_t(\mathbf{x})+\beta_t\,\sigma_t(\mathbf{x})\big]
\)
is amplified near high-credit areas.
On the high-probability event, this mechanism makes the sampling tendency toward high-value regions, which promotes earlier hits near the optimum and thus faster early simple-regret decay.
\end{remark}

We then analyze the cumulative regret of the Credit-Weighted UCB variant of CCGBO.  Our strategy closely follows the classic GP‐UCB analysis of Srinivas \emph{et al.}~\citep{srinivas2009gaussian} and~\citep{hvarfner2022pi}, with the key modification that at each iteration \(t\) the standard UCB acquisition is rescaled by a factor
\(
w_t(\mathbf{x}) = \bigl[c(\mathbf{x})\bigr]^{\frac{\tau}{\,1 + (t/M)}},
\)
where \(c(\mathbf{x})\) is the normalized counterfactual credit and \(\tau>0\) is the decay parameter. We show that this rescaling only incurs a constant multiplicative penalty in the cumulative regret bound.

\textbf{Assumption 2} (RKHS regularity)\label{assump:rkhs}
The objective function \(f\) lies in the reproducing kernel Hilbert space (RKHS) \(\mathcal{H}_k\) associated with kernel \(k\), with \(\|f\|_{\mathcal{H}_k} \le B\) and \(k(\mathbf{x},\mathbf{x}) \le 1\) for all \(\mathbf{x} \in \mathcal{X}\). By the reproducing property, this implies the deterministic bound \(|f(\mathbf{x})| \le B\) for all \(\mathbf{x} \in \mathcal{X}\).

Under this assumption, following Theorem~3 of~\citet{srinivas2009gaussian}, we choose \(\beta_t\) such that
\(
|f(\mathbf{x})-\mu_{t-1}(\mathbf{x})| \;\le\; \beta_t\,\sigma_{t-1}(\mathbf{x})
\)
holds with probability at least \(1-\delta\) over the observation noise.

We define the cumulative regret of CCGBO as
\(
  R_N^{\mathrm{ccg}} = \sum_{t=1}^N r_t^{\mathrm{ccg}}, \quad \text{where } r_t^{\mathrm{ccg}} = f(\mathbf{x}^*) - f(\mathbf{x}_t)
\)
is the instantaneous regret at iteration \(t\).

\begin{theorem}\label{thm:ccgbo}
Define for each \(t\),
\(
A_t = 1-\lambda + \lambda\,c_{\min}^{\,\tau/(1 + t/M)},
~
B_t = 1-\lambda + \lambda\,c_{\max}^{\,\tau/(1 + t/M)}.
\)
Observe that \(0<A_t\le B_t<\infty\) and \(B_t/A_t\to1\) as \(t\to\infty\).
Under Assumption 2 and the above high‐probability event, the cumulative regret of CCGBO satisfies
\begin{equation}
R_N^{\rm ccg}
\;\le\;
\max_{1\le t\le N}\!\Bigl(\tfrac{B_t}{A_t}\Bigr)\;R_N,
\end{equation}
where \(R_N=O\bigl(\sqrt{N\,\gamma_N\,\beta_N}\bigr)\) is the standard GP-UCB bound from~\citep{srinivas2009gaussian} and \(\gamma_N\) is the maximum information gain.  In particular, since \(\max_{t}B_t/A_t=B_1/A_1<\infty\), CCGBO matches the vanilla GP-UCB rate up to a constant factor that tends to 1 when \(\lambda\to0\) (so \(A_t=B_t=1\)) or \(c_{\max}/c_{\min}\to1\).
\end{theorem}

\begin{remark}[Practical Interpretation of the Constant Factor]
The constant factor \(\max_{1\le t\le N}(B_t/A_t)\) in Theorem~\ref{thm:ccgbo} arises from a conservative worst-case analysis. In practice, this factor is benign for two reasons:
(1)~The ratio \(B_t/A_t\) is monotonically decreasing in \(t\) and converges to 1 as \(t\to\infty\), since the exponent \(\tau/(1+t/M)\to 0\) causes both \(c_{\min}^{\tau/(1+t/M)}\) and \(c_{\max}^{\tau/(1+t/M)}\) to approach 1.
(2)~The worst-case factor \(B_1/A_1\) occurs only at \(t=1\) which is a modest overhead.
Thus, while the theoretical bound includes this constant, the credit weighting effectively degenerates in later iterations, allowing CCGBO to inherit the standard GP-UCB convergence behavior asymptotically.
\end{remark}

The cumulative regret of the standard GP–UCB algorithm is known to satisfy
\(
  R_N = O\bigl(\sqrt{N\,\gamma_N\,\beta_N}\bigr)\,.
\)
In contrast, for CCGBO we obtain the bound
\(
R_N^{\rm ccg}
\;\le\;
\max_{1\le t\le N}\!\Bigl(\tfrac{B_t}{A_t}\Bigr)\;R_N.
\)
Thus, the regret of CCGBO exceeds that of GP–UCB by at most the constant factor \(\max_t(B_t/A_t)\), which is controlled by the extremal credit weights \(c_{\min},c_{\max}\) and the decay parameters \(\tau\) and \(M\).
The proof proceeds by comparing the two acquisition values under the high-probability confidence event, introducing \(A_t\) and \(B_t\) to extract the influence of the credit weights from the regret analysis. 

The theoretical analysis shows that credits computed from $Z_t$ are near the optimal region with high possibility, and also the counterfactual credits does not destroy the sublinear convergence rate of GP-UCB. CCGBO preserves the sublinear convergence rate of GP–UCB up to a controllable constant, demonstrating that adding counterfactual credit only incurs a slight overhead at the worst-case constant level.


\section{Experiments}
We validated CCGBO's performance against baseline methods on both synthetic functions and real-world tasks. The following are the experimental setup and results on the benchmarks.

\begin{figure*}[!t]
    \centering
    \begin{minipage}[b]{1\textwidth}
        \centering
        \includegraphics[width=1\textwidth]{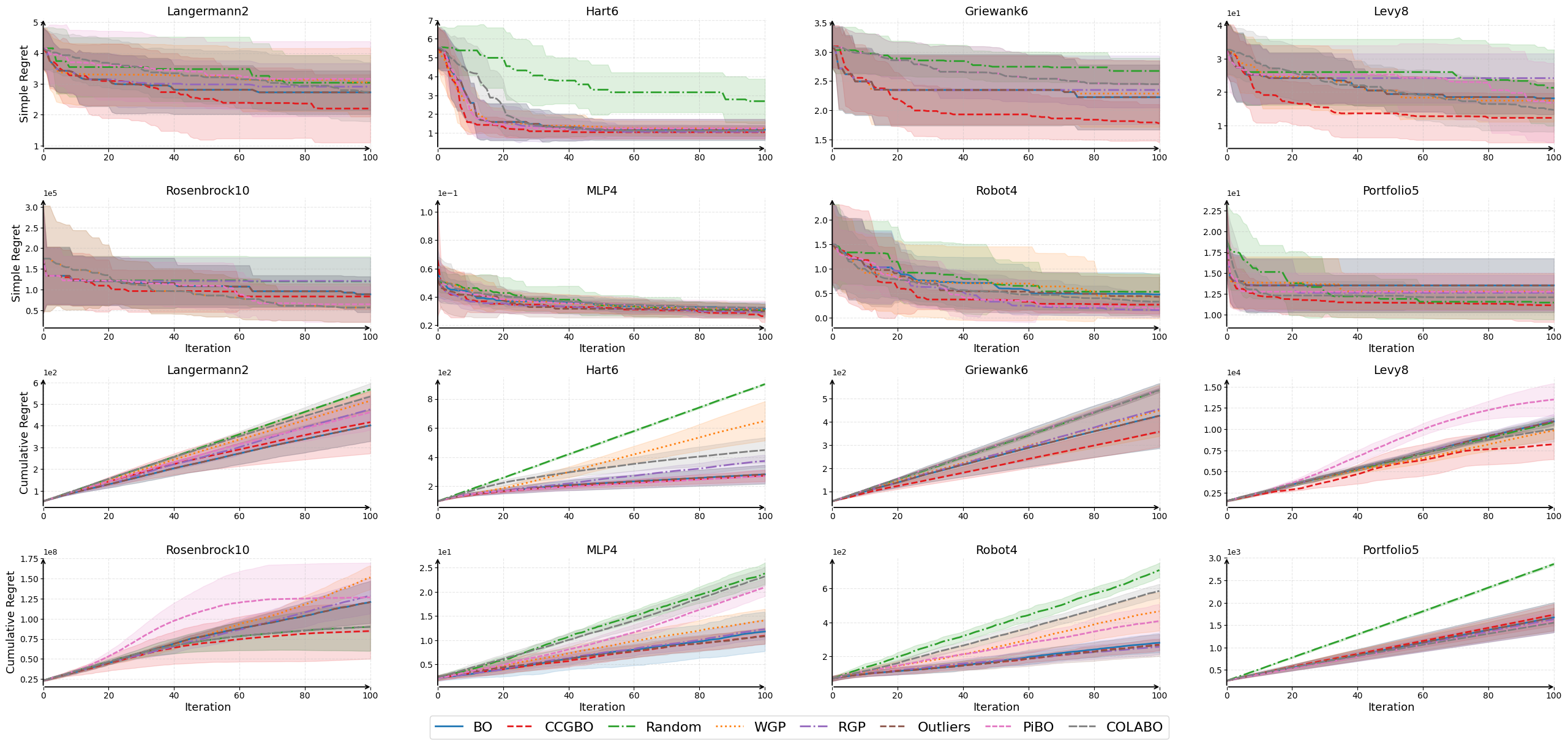}
    \end{minipage}
    \caption{Cumulative regret and Simple regret versus iteration for eight benchmark functions.}
    \label{fig:Result_UCB}
\end{figure*}

\begin{table*}[h]
\centering
\caption{Area under simple regret metric across benchmarks (lower is better).}
\label{tab:ausr}
\scriptsize
\setlength{\tabcolsep}{3pt}
\renewcommand{\arraystretch}{1.05}
\resizebox{\linewidth}{!}{%
\begin{tabular}{l*{8}{P}}
\toprule
\textsc{Task} & \multicolumn{2}{c}{\textsc{BO}} & \multicolumn{2}{c}{\textsc{CCGBO}} &
\multicolumn{2}{c}{\textsc{Random}} & \multicolumn{2}{c}{\textsc{WGP}} &
\multicolumn{2}{c}{\textsc{RGP}} & \multicolumn{2}{c}{\textsc{Outlier}} &
\multicolumn{2}{c}{\textsc{PiBO}} & \multicolumn{2}{c}{\textsc{COLABO}} \\
\midrule
Langermann2  & \val{292.3}{69.0}  & \best{268.4}{103.0} & \val{338.3}{84.5}  & \val{324.0}{94.7}  & \val{304.7}{73.0}  & \val{292.3}{69.0}  & \val{339.6}{103.1} & \val{329.9}{71.2} \\
Hart6        & \val{160.3}{58.5}  & \best{134.9}{27.7}  & \val{376.5}{80.2}  & \val{160.7}{42.2}  & \val{152.5}{47.5}  & \val{160.3}{58.5}  & \val{157.3}{44.1}  & \val{181.1}{37.3} \\
Griewank6    & \val{235.0}{52.0}  & \best{203.3}{35.5}  & \val{280.6}{23.4}  & \val{236.5}{54.1}  & \val{238.3}{55.8}  & \val{235.0}{52.0}  & \val{266.0}{31.1}  & \val{266.0}{31.1} \\
Levy8        & \val{2128.3}{510.4} & \best{1514.8}{620.3} & \val{2530.9}{963.2} & \val{2116.1}{563.2} & \val{2446.2}{735.0} & \val{2128.3}{510.4} & \val{2368.8}{930.9} & \val{2132.4}{536.3} \\
Rosenbrock10 ($1e4$)
             & \val{1085.7}{436.6} & \val{966.1}{394.6}  & \val{1221.9}{574.5} & \val{948.0}{406.6}  & \val{1218.1}{559.9} & \val{1085.7}{436.6} & \val{969.7}{428.3} & \best{947.9}{406.6} \\
MLP4         & \val{3.6}{0.3}     & \best{3.3}{0.4}     & \val{3.7}{0.6}     & \val{3.6}{0.5}     & \best{3.3}{0.3}     & \val{3.4}{0.6}     & \val{3.6}{0.5}     & \val{3.7}{0.3} \\
Robot4       & \val{71.1}{39.2}    & \best{48.3}{25.9}    & \val{80.6}{42.8}    & \val{71.1}{59.8}    & \val{50.6}{23.9}    & \val{66.0}{36.4}    & \val{52.2}{24.1}    & \val{60.7}{26.4} \\
Portfolio5   & \val{1360.3}{304.7} & \best{1170.3}{171.3} & \val{1277.0}{225.4} & \val{1320.2}{218.5} & \val{1290.1}{214.6} & \val{1360.3}{304.7} & \val{1287.6}{223.3} & \val{1244.4}{154.9} \\
\bottomrule
\end{tabular}}
\end{table*}

\subsection{Experimental Setup}

We include the following baselines:
First, Standard GP‑UCB~\citep{srinivas2009gaussian} and Random Search are the most natural baselines for BO.
Then, we include Weighted Gaussian Process UCB (WGP)~\citep{deng2022weighted} and Time-Varying Gaussian Process UCB (RGP)~\citep{bogunovic2016time} to cover the non‐stationary scenario. By comparing with WGP and RGP, we show that CCGBO's credit mechanism can also eliminate low-value observations during the optimization process.
Next, OutlierBO~\citep{martinez2018practical} represents robust outlier handling. In contrast, CCGBO's credit assignment naturally downweights low‐contribution points (including noise and outliers).
Finally, PiBO~\citep{hvarfner2022pi} and ColaBO~\citep{hvarfner2023general} embody user‐prior, requiring expert‐supplied beliefs about the optimum. By comparing with PiBO and COLABO, we highlight that our approach can achieve competitive results without any external prior.
WGP, RGP, and OutlierBO are grouped as knowledge from previous experiments methods, whereas PiBO and ColaBO belong to Embedding Priors over the Function Optimum methods.
We demonstrate the superiority of our algorithm by comparing it with standard, non-stationary, robust, and prior guidance methods.
Note on PiBO and ColaBO, for a fair comparison, we do not elicit any external priors from users. Instead, we construct the user‐provided prior by taking the location of the best point in the initial set as the prior mean for both PiBO and ColaBO.

In all experiments, each function evaluation is corrupted by independent Gaussian noise $\varepsilon\sim\mathcal N(0,\sigma_n^2)$ with $\sigma_n^2=0.01$, to simulate realistic measurements. The size of the initial design is set adaptively to \(n_{\mathrm{init}}= \max\bigl(2\,d,\;10\bigr)\,,\) where $d$ denotes the dimensionality of the problem. 
We employ a Mat\'ern $5/2$ kernel for the GP, with hyperparameters learned via maximum likelihood estimation. The iteration budget is 100. The regret is averaged on 50 independent experiments. For the acquisition function, we set the UCB exploration parameter $\beta_t = 2.576$, corresponding to the $99\%$ quantile of a standard normal distribution, and the credit influence parameter $\lambda = 0.5$. We set the ``half‐life'' parameter $M=20$, the Monte Carlo UCB proxy sample number $K=25$, and the number of nearest neighbors \(H=5\).
We employed eight benchmarks, five synthetic and three real-world problems. For clarity, all the tasks aim for function maximization. All experiments were conducted on a MacBook Pro with Apple M2 Pro (10-core CPU, 16 GB RAM). Our code is publicly available in the repository: \url{https://github.com/Qiyu-Wei/CCGBO}.

\textbf{Synthetic Test Problems.}
We evaluate five synthetic test functions:  
\emph{Langermann2} is a 2-dimensional multimodal function on \([0,10]^2\);  
\emph{Hartmann6} is a six-dimensional function on \([0,1]^6\);  
\emph{Griewank6} is a nonconvex six-dimensional landscape defined on \([-600,600]^6\);  
\emph{Levy8} is an 8-dimensional valley-shaped function on \([-10,10]^8\);
and \emph{Rosenbrock10} is a 10-dimensional test function on \([-5,10]^{10}\). We put high-dim synthetic test problems in Supplementary Materials. 

\textbf{Real-World Test Problems.}
We further evaluated CCGBO on three real-world problems.
(1) The NAS experiment models a breast cancer prediction problem from the UCI repository~\citep{Dua:2019} of searching for optimal hyperparameters as a BO problem. Specifically, each query $(x_t,y_t)$ corresponds to a choice of ($[32, 128], [1e-6, 1.0], [1e-6, 1.0],[1, 8]$ for batch size, Learning Rate, Learning Rate decay, hidden dim, respectively), where $y_t$ is the test classification error. (2) The Robot4~\citep{kaelbling2017learning} problem simulates a 4-dimensional robot pushing task using the Box2D physics engine, where each query corresponds to robot initial position $ \in [-5,5]^2$ , initial robot angle control $\in [0, 2\pi]$ and simulation steps $\in [1,30]$, with $y_t$ being the Euclidean distance between the pushed object's final position and target. (3) The Portfolio problem~\citep{bruni2016real} represents a 5-dimensional financial portfolio optimization task, where each query corresponds to asset normalized weight allocations $\in [0,1]^5$.

\subsection{Comparison of CCGBO Against Baselines}

Figures~\ref{fig:Result_UCB} summarize the results on eight benchmarks, and show the mean and standard error of the simregret and cumregret on 8 benchmark problems. We also report the area under simple regret in Table~\ref{tab:ausr} as a summary metric of the entire optimization trajectory, defined as
\(
\mathrm{AUSR}
= \frac{1}{T-1}\sum_{t=2}^{T} \frac{r_{t-1}+r_{t}}{2},
\)
where \(r_t\) denotes the simple regret at iteration \(t\). It measures the remaining gap between the best value found so far and the true optimum.
We report the observations below:

\textbf{(1) Faster regret convergence:}  Across all benchmarks, CCGBO attains the fastest drop of simple regret especially in the early stage and settling at the low regret plateau, outperforming Standard GP‐UCB. This is also confirmed in the results reported in the Table~\ref{tab:ausr}.
This shows that compared with the baseline method, counterfactual credit weighting can more effectively prioritize high-value areas and avoid redundant exploration in ``less promising'' areas. Especially when there are large flat areas or multiple local extremes in the objective function, counterfactual credit can quickly focus on the most informative sampling points, thereby significantly accelerating the convergence speed. Meanwhile, because CCGBO established advantages in the early stage, it has a better understanding of the global optimum in the early stage, which is helpful for the subsequent optimization process, and simple regret maintained an advantage at most tasks.

\textbf{(2) Lower cumulative regret:}  CCGBO maintains an advantage over most baselines, while random search quickly accumulates regret and specialized schemes like WGP and RGP only partially close the gap. This confirms that CCGBO preserves the sublinear regret rate while delivering practical gains.

\textbf{(3) Robustness without artificial priors:}  Unlike PiBO and ColaBO, which rely on user‐supplied beliefs, CCGBO achieves equal or better performance without external priors. The optimization effects of PiBO and ColaBO vary on different benchmarks due to the fact that the prior settings are sometimes good and sometimes bad. Moreover, our method outperforms Outlier‐Robust BO even in the noisy setting, since low‐credit observations are naturally downweighted by the counterfactual credit mechanism.


Our algorithm is plug-and-play, it also achieves good results on other different acquisition functions and high-dimensional problems. We put the experiment on other acquisition functions (TS, EI, JES etc.), high-dim problems (25 to 1000 dimensions) and ablation study on $M$, $K$, $H$ in Supplementary Materials. Due to space limitations, details are provided in Supplementary Materials.

Together, these results confirm that CCGBO delivers a general improvement over existing BO techniques, accelerating convergence, reducing cumulative loss, and obviating the need for user-specified priors. It should be noted that the time spent on calculating credit is very small compared to the optimization process and can be ignored.

\section{Conclusion}

In this paper, we introduced CCGBO, an efficient framework that integrates counterfactual credit assignments into the acquisition process to explicitly quantify and exploit the varying contribution of past observations. By augmenting the traditional exploration–exploitation trade-off with a third importance dimension, our method adaptively focuses sampling effort on regions most likely to yield rapid improvements toward the global optimum. We derived an estimator for per-sample credit based on the GP posterior and a Monte Carlo proxy, incorporated this into a Credit-Weighted UCB acquisition function, and showed that the resulting algorithm retains faster simple regret convergence and sublinear cumulative regret. Empirical results on synthetic test function benchmarks and real-world hyperparameter‐tuning task demonstrate that CCGBO outperforms standard BO and recent baseline alternatives, especially in faster convergence rate.

\bibliography{main}


\newpage
\appendix
\onecolumn

\clearpage

\clearpage
\appendix
\thispagestyle{empty}

\onecolumn
\aistatstitle{Supplementary Materials}

\section{Iteration Illustration of CCGBO}

\begin{figure}[!ht]
    \centering
    \begin{minipage}[b]{\textwidth}
        \centering
        \includegraphics[width=0.6\textwidth]{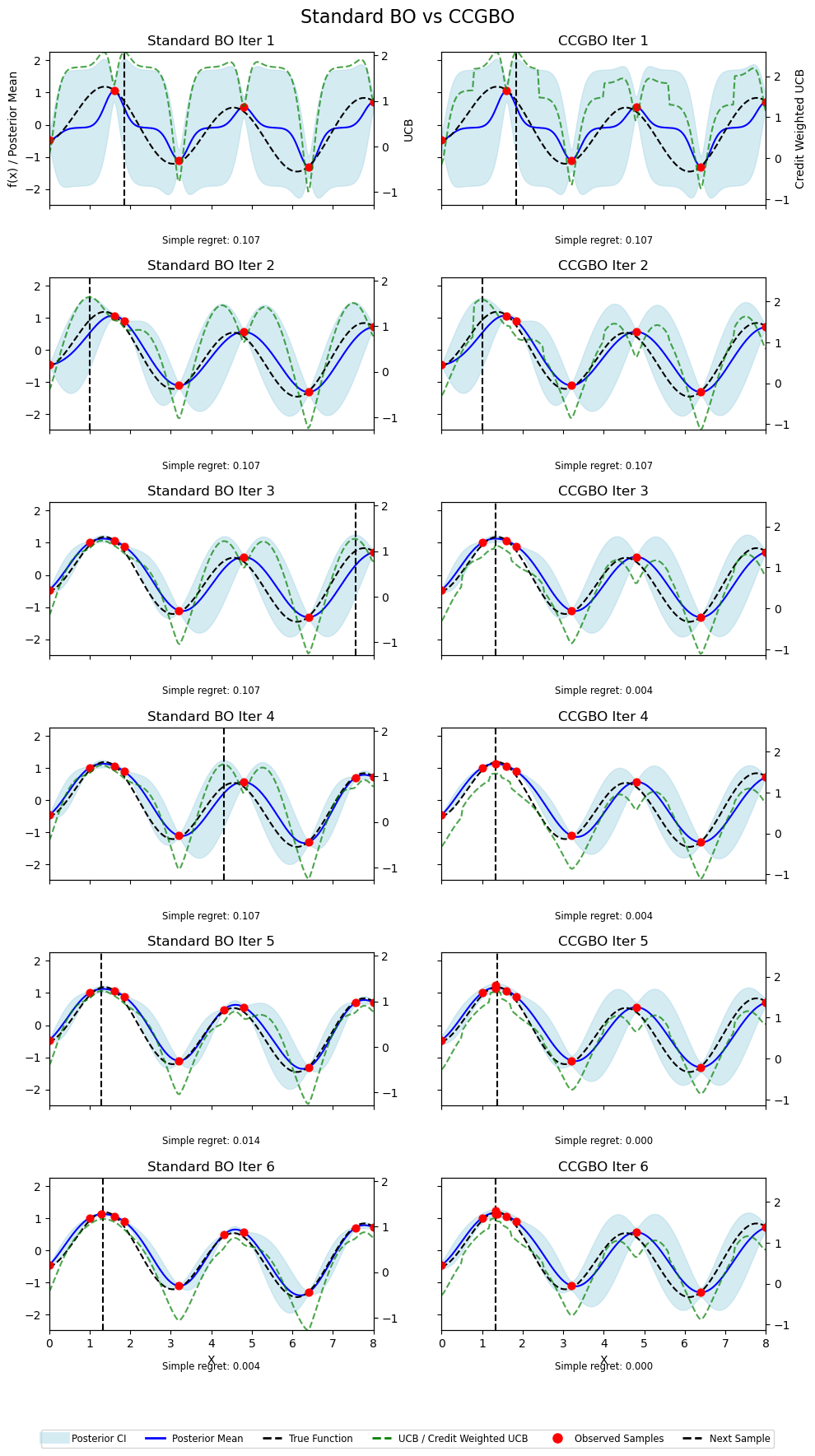}
    \end{minipage}
    \caption{%
    \textbf{Sequential optimization steps for Standard BO (left) vs.\ CCGBO (right).}  
    Each row shows iterations 1 through 6 on a one-dimensional toy function.  
    Blue curves and shaded regions: GP posterior mean and $95\%$ credible interval.  
    Black dashed curve: true objective.  
    Green dashed line: standard UCB (\emph{left}) or credit‐weighted UCB (\emph{right}).  
    Red dots: observed samples.  
    Vertical dashed line: next evaluation location.  
    Simple regret is reported below each subplot.  
    }
    \label{fig:iter_visualization}
\end{figure}

Figure~\ref{fig:iter_visualization} shows that CCGBO locates the global optimum far more quickly than standard GP‐UCB, reflected in a steeper decline of simple regret by iteration 4.  Although CCGBO's posterior fit in non‐critical regions can be less precise than that of standard UCB, our method deliberately allocates more evaluation budget to the high‐value region.  This focused sampling produces a more accurate local model around the true maximum, which in turn drives significantly better optimization performance.

\section{Proofs}
Here, we provide the complete proofs for the Theorem introduced in Section 5.

\subsection{Consistency of the MC proxy of the optimum}

\textbf{Assumptions 1}
There exists $\beta_t>0$ such that the GP-UCB confidence event
\(
\mathcal E_t:=\Big\{\,\lvert f(\mathbf{x})-\mu_t(\mathbf{x})\rvert \le \beta_t\,\sigma_t(\mathbf{x})\ \ \forall \mathbf{x}\in\mathcal X\,\Big\}
\)
holds with probability at least $1-\delta$~\citep{srinivas2009gaussian}. And there exist constants $c_0,c_1>0$ such that for the zero-mean posterior process 
$\varepsilon_t:=f_t-\mu_t$ we have
$\mathbb{E}\!\big[\sup_{\mathbf{x}\in\mathcal X}\varepsilon_t(\mathbf{x})\,\bigm|\,\mathcal D_t\big]\le c_0 S_t$~\citep{adler2007random},
and, conditionally on $\mathcal D_t$, the centered variables $Z_t^{(j)}-\mathbb{E}[Z_t^{(j)}\mid\mathcal D_t]$
are $(c_1 S_t)$-sub-Gaussian (equivalently, their $\psi_2$-norms are $\le c_1 S_t$)~\citep{ledoux2013probability}.

\begin{theorem}\label{prop:mc-consistency2}
Let $\mathcal X\subset\mathbb{R}^d$ be compact, and at iteration $t$ let the GP posterior
be $f_t\sim\mathcal{GP}(\mu_t,k_t)$ with $\sigma_t^2(\mathbf{x})=k_t(\mathbf{x},\mathbf{x})$ and
$S_t=\sup_{\mathbf{x}\in\mathcal X}\sigma_t(\mathbf{x})<\infty$.
Draw $K$ i.i.d.\ posterior sample paths $\{f_t^{(j)}\}_{j=1}^K$, and define
$Z_t^{(j)}=\max_{\mathbf{x}\in\mathcal X} f_t^{(j)}(\mathbf{x})$ and $Z_t=\tfrac1K\sum_{j=1}^K Z_t^{(j)}$.
Let $\mathbf{x}^\star\in\arg\max_{\mathbf{x}\in\mathcal X} f(\mathbf{x})$.
For any $\delta'\in(0,1)$, there exists a constant $C>0$ (depending only on $c_1$) such that
\begin{equation}\label{eq:mc-consistency-bound2}
\Pr\!  \left( 
  \bigl|Z_t - f(\mathbf{x}^\star)\bigr| 
  \le 
  \underbrace{\beta_t S_t}_{\text{model bias}}
  +
  \underbrace{C\,S_t\,\sqrt{\tfrac{\log(1/\delta')}{K}}}_{\text{MC error}}
\right)\;  \ge\; 1-\delta-\delta'.
\end{equation}
\end{theorem}

\begin{proof}
By Assumptions 1. Since $\mathcal X$ is compact, maxima exist. Decompose $\lvert Z_t-f(\mathbf{x}^\star)\rvert\le T_1+T_2$, where
$T_1=\lvert Z_t-\mathbb{E}[Z_t^{(j)}\mid\mathcal D_t]\rvert$ and
$T_2=\lvert \mathbb{E}[Z_t^{(j)}\mid\mathcal D_t]-f(\mathbf{x}^\star)\rvert$.
By standard sub-Gaussian concentration for sample means (conditionally on $\mathcal D_t$),
$T_1\le C S_t \sqrt{\log(1/\delta')/K}$ with conditional probability $\ge 1-\delta'$.
Next, write $f_t^{(j)}=\mu_t+\varepsilon_t^{(j)}$ with zero-mean Gaussian $\varepsilon_t^{(j)}$;
then
\begin{equation}
\mathbb{E}[Z_t^{(j)}\mid\mathcal D_t]
=\mathbb{E}\Big[\ \sup_{\mathbf{x}\in\mathcal X}\big(\mu_t(\mathbf{x})+\varepsilon_t^{(j)}(\mathbf{x})\big)\ \Bigm|\ \mathcal D_t\Big].
\end{equation}
Using $\sup(a+b)\le \sup a+\sup b$, symmetry of $\varepsilon_t^{(j)}$ (since it is zero-mean Gaussian), we get
$\big\lvert \mathbb{E}[Z_t^{(j)}\mid\mathcal D_t]-\sup_\mathbf{x} \mu_t(\mathbf{x})\big\rvert \le c_0 S_t$.
On $\mathcal E_t$, $\big\lvert \sup_\mathbf{x}\mu_t(\mathbf{x})-f(\mathbf{x}^\star)\big\rvert\le \beta_t S_t$, hence
$T_2\le (\beta_t+c_0)S_t$; absorbing $c_0$ into $\beta_t$ yields $T_2\le \beta_t S_t$.
Finally, a union bound gives total probability at least $1-\delta-\delta'$, which proves \eqref{eq:mc-consistency-bound}.
\end{proof}

\begin{remark}[RKHS Setting]
Theorem~\ref{prop:mc-consistency} remains valid under the RKHS setting of Assumption 2. The decomposition into \(T_1\) (Monte Carlo error) and \(T_2\) (model bias) still applies: \(T_1\) depends solely on the GP surrogate model, while the bound on \(T_2\) follows from the confidence event which holds under RKHS assumptions per Theorem~3 of~\citet{srinivas2009gaussian}.
\end{remark}

\subsection{Regret Bound}


We denote the instantaneous regret at step \(t\) by
\(
r_t = f(\mathbf{x}^*) - f(\mathbf{x}_t),
\)
where \(\mathbf{x}^*=\arg\max_{\mathbf{x}}f(\mathbf{x})\), and the cumulative regret by
\(
R_N = \sum_{t=1}^N r_t.
\)
We will analyze the Credit-Weighted UCB (CCGBO) acquisition
\(
\alpha_t(\mathbf{x}) = \bigl(1-\lambda + \lambda\,w_t(\mathbf{x})\bigr)\,\bigl[\mu_{t-1}(\mathbf{x}) + \beta_t\,\sigma_{t-1}(\mathbf{x})\bigr],
\)
where
\(
w_t(\mathbf{x}) = \bigl[c(\mathbf{x})\bigr]^{\tfrac{\tau}{1 + t/M}},
\)
and we further assume
\(
0 < c_{\min} \le c(\mathbf{x}) \le c_{\max} < \infty
~\forall \mathbf{x}\in\mathcal X,
\)
\(\lambda\in[0,1]\) is the credit‐influence parameter, and \(\{\beta_t\}\) is chosen as in~\citet{srinivas2009gaussian}. Under Assumption 2, the confidence bound \(|f(\mathbf{x})-\mu_{t-1}(\mathbf{x})| \le \beta_t\,\sigma_{t-1}(\mathbf{x})\) holds with probability at least \(1-\delta\).

\begin{theorem}\label{thm:ccgbo2}
Define for each \(t\),
\(
A_t = 1-\lambda + \lambda\,c_{\min}^{\,\tau/(1 + t/M)},
~
B_t = 1-\lambda + \lambda\,c_{\max}^{\,\tau/(1 + t/M)}.
\)
Observe that \(0<A_t\le B_t<\infty\) and \(B_t/A_t\to1\) as \(t\to\infty\).
Under Assumption 2 and the above high‐probability event, the cumulative regret of CCGBO satisfies
\begin{equation}
R_N^{\rm ccg}
\;\le\;
\max_{1\le t\le N}\!\Bigl(\tfrac{B_t}{A_t}\Bigr)\;R_N,
\end{equation}
where \(R_N=O\bigl(\sqrt{N\,\gamma_N\,\beta_N}\bigr)\) is the standard GP-UCB bound from~\citep{srinivas2009gaussian} and \(\gamma_N\) is the maximum information gain.  In particular, since \(\max_{t}B_t/A_t=B_1/A_1<\infty\), CCGBO matches the vanilla GP-UCB rate up to a constant factor that tends to 1 when \(\lambda\to0\) (so \(A_t=B_t=1\)) or \(c_{\max}/c_{\min}\to1\).
\end{theorem}

\begin{proof}
Condition on the event \(\mathcal E\) that for all \(t,\mathbf{x}\),
\(
f(\mathbf{x})\le\mu_{t-1}(\mathbf{x})+\beta_t\,\sigma_{t-1}(\mathbf{x}),
~
f(\mathbf{x})\ge\mu_{t-1}(\mathbf{x})-\beta_t\,\sigma_{t-1}(\mathbf{x}).
\)
Define the shifted acquisition function. Let
\(
B_{\min,t} := \min_{\mathbf{x} \in \mathcal{X}} \bigl[\mu_{t-1}(\mathbf{x}) + \beta_t\,\sigma_{t-1}(\mathbf{x})\bigr],
\)
and set \(\tilde{\alpha}_t(\mathbf{x}) = \mu_{t-1}(\mathbf{x}) + \beta_t\,\sigma_{t-1}(\mathbf{x}) - B_{\min,t} \ge 0\). The credit-weighted acquisition becomes
\(
a_t(\mathbf{x}) = \bigl(1-\lambda + \lambda\,w_t(\mathbf{x})\bigr)\,\tilde{\alpha}_t(\mathbf{x}).
\)
By definition, \(\mathbf{x}_t=\arg\max_{\mathbf{x}} a_t(\mathbf{x})\), so for the true maximizer \(\mathbf{x}^*\),
\(
a_t(\mathbf{x}^*) \le a_t(\mathbf{x}_t).
\)
Thus
\begin{equation}
\begin{split}
&A_t\bigl[\mu_{t-1}(\mathbf{x}^*)+\beta_t\,\sigma_{t-1}(\mathbf{x}^*) - B_{\min,t}\bigr]\\
&\le
\bigl(1-\lambda + \lambda\,w_t(\mathbf{x}^*)\bigr)\bigl[\mu_{t-1}(\mathbf{x}^*)+\beta_t\,\sigma_{t-1}(\mathbf{x}^*) - B_{\min,t}\bigr]\\
&= a_t(\mathbf{x}^*)
\le
a_t(\mathbf{x}_t)
\le
B_t\bigl[\mu_{t-1}(\mathbf{x}_t)+\beta_t\,\sigma_{t-1}(\mathbf{x}_t) - B_{\min,t}\bigr].
\end{split}
\end{equation}
Rearranging gives
\(
\mu_{t-1}(\mathbf{x}^*)+\beta_t\,\sigma_{t-1}(\mathbf{x}^*) - B_{\min,t}
\le
\Bigl(\tfrac{B_t}{A_t}\Bigr)\bigl[\mu_{t-1}(\mathbf{x}_t)+\beta_t\,\sigma_{t-1}(\mathbf{x}_t) - B_{\min,t}\bigr].
\)
Since \(f(\mathbf{x}^*) \le \mu_{t-1}(\mathbf{x}^*)+\beta_t\,\sigma_{t-1}(\mathbf{x}^*)\) and \(f(\mathbf{x}_t)\ge\mu_{t-1}(\mathbf{x}_t)-\beta_t\,\sigma_{t-1}(\mathbf{x}_t)\), we obtain
\begin{equation}
\begin{aligned}
r_t^{\mathrm{ccg}}
&= f(\mathbf{x}^*) - f(\mathbf{x}_t)\\
&= \bigl(f(\mathbf{x}^*) - B_{\min,t}\bigr) - \bigl(f(\mathbf{x}_t) - B_{\min,t}\bigr)\\
&\le \bigl[\mu_{t-1}(\mathbf{x}^*)+\beta_t\,\sigma_{t-1}(\mathbf{x}^*) - B_{\min,t}\bigr] - \bigl[f(\mathbf{x}_t) - B_{\min,t}\bigr]\\
&\le \Bigl(\tfrac{B_t}{A_t}\Bigr)\bigl[\mu_{t-1}(\mathbf{x}_t)+\beta_t\,\sigma_{t-1}(\mathbf{x}_t) - B_{\min,t}\bigr] - \bigl[\mu_{t-1}(\mathbf{x}_t)-\beta_t\,\sigma_{t-1}(\mathbf{x}_t) - B_{\min,t}\bigr]\\
&= \underbrace{\Bigl(\tfrac{B_t}{A_t}\Bigr)\,2\beta_t\,\sigma_{t-1}(\mathbf{x}_t)}_{\text{standard term}}
   + \underbrace{\Bigl(\tfrac{B_t}{A_t}-1\Bigr)\bigl[\mu_{t-1}(\mathbf{x}_t)-\beta_t\,\sigma_{t-1}(\mathbf{x}_t)\bigr]}_{\text{residual term}}
   - \underbrace{\Bigl(\tfrac{B_t}{A_t}-1\Bigr) B_{\min,t}}_{\text{shift term}}.
\end{aligned}
\end{equation}

For the standard term, note that $c_{\min}>0$ and $c_{\max}\le 1$ imply that $A_t$ and $B_t$
are uniformly bounded away from $0$ and $\infty$, so there exists a constant $C_1>0$ such that
$\frac{B_t}{A_t}\le C_1$ for all $t$. Hence
\[
\frac{B_t}{A_t}\,2\beta_t\sigma_{t-1}(\mathbf{x}_t)
\;\le\; 2C_1\,\beta_t\sigma_{t-1}(\mathbf{x}_t).
\]
Summing over $t$ and using the standard GP-UCB argument of \citet{srinivas2009gaussian}, we obtain
\[
\sum_{t=1}^N \frac{B_t}{A_t}\,2\beta_t\sigma_{t-1}(\mathbf{x}_t)
= O\bigl(\sqrt{N\,\gamma_N\,\beta_N}\bigr).
\]

For the residual and shift terms, observe that $\mu_{t-1}(\mathbf{x}_t)-\beta_t\sigma_{t-1}(\mathbf{x}_t)$ is the lower confidence bound. Under Assumption 2, we have $|f(\mathbf{x})|\le B$ for all $\mathbf{x}$, and on the confidence event
\[
\mu_{t-1}(\mathbf{x}_t)-\beta_t\sigma_{t-1}(\mathbf{x}_t)
\;\le\; f(\mathbf{x}_t) \;\le\; B.
\]
For the shift term, note that on the confidence event $\mu_{t-1}(\mathbf{x})+\beta_t\sigma_{t-1}(\mathbf{x}) \ge f(\mathbf{x}) \ge -B$ for all $\mathbf{x}$. Hence $B_{\min,t} \ge -B$, which implies $-B_{\min,t} \le B$.

Moreover, by construction $B_t\ge A_t$, so $\frac{B_t}{A_t}-1\ge 0$.
From the explicit expressions
\[
A_t = 1-\lambda + \lambda c_{\min}^{\alpha_t},\qquad
B_t = 1-\lambda + \lambda c_{\max}^{\alpha_t},\qquad
\alpha_t = \frac{\tau}{1+t/M},
\]
and the fact that $c_{\min},c_{\max}\in(0,1]$, a simple application of the mean value theorem
shows that there exists a constant $C_2>0$ such that
\[
\frac{B_t}{A_t}-1 \;\le\; \frac{C_2}{1+t/M}
\quad\text{for all }t.
\]
Therefore, combining the residual and shift terms and using the bounds above:
\[
\sum_{t=1}^N\Bigl(\frac{B_t}{A_t}-1\Bigr)\bigl[\mu_{t-1}(\mathbf{x}_t)-\beta_t\sigma_{t-1}(\mathbf{x}_t)\bigr]
- \sum_{t=1}^N\Bigl(\frac{B_t}{A_t}-1\Bigr) B_{\min,t}
\;\le\; \sum_{t=1}^N 2\Bigl(\frac{B_t}{A_t}-1\Bigr) B
\;\le\; 2 C_2 B \sum_{t=1}^N \frac{1}{1+t/M}
= O(\log N).
\]

Combining the bounds for the standard, residual, and shift terms, we get
\[
R_N^{\mathrm{ccg}}
=\sum_{t=1}^N r_t^{\mathrm{ccg}}
= O\bigl(\sqrt{N\,\gamma_N\,\beta_N}\bigr) + O(\log N)
= O\bigl(\sqrt{N\,\gamma_N\,\beta_N}\bigr),
\]
so CCGBO retains the same sublinear regret rate as GP-UCB.

\end{proof}

\section{CCGBO with other acquisition functions}

\paragraph{Relation to Information-based Acquisition Functions.}
Information-based acquisition functions, such as Entropy Search (ES)~\citep{hennig2012entropy}, Predictive Entropy Search (PES)~\citep{hernandez2014predictive}, and Max-value Entropy Search (MES)~\citep{wang2017max}, focus on selecting points that maximize the information gain about the location of the optimum $\mathbf{x}^*$. While conceptually related to our work in seeking to identify promising regions, CCGBO differs in two key aspects: (1)~Our core principle is to use the function value proxy $Z_t$ to compute counterfactual credit for reweighting the acquisition function, rather than directly modeling the distribution of $\mathbf{x}^*$; information-based methods still treat all candidate points equally when computing information gain. (2)~CCGBO is a generic upper-layer module that can be combined with any base acquisition function, including ES-family methods, whereas ES/PES/MES are standalone acquisition functions. We demonstrate this compatibility by combining CCGBO with Joint Entropy Search (JES)~\citep{hvarfner2022joint}, an advanced member of the ES family, and observe consistent improvements (see Figure~\ref{fig:Result_JES}).

\paragraph{Compatibility with Different Acquisition Functions.}
CCGBO is a plug-and-play module and can be combined with acquisition functions beyond UCB. The counterfactual credit-based weight can be simply multiplied by other acquisition functions as a coefficient. As suggested, we have evaluated it with several others, including TS in Figure~\ref{fig:Result_TS}, logEI~\citep{ament2023unexpected} in Figure~\ref{fig:Result_EI}, and JES~\citep{hvarfner2022joint} in Figure~\ref{fig:Result_JES}. We observe that CCGBO yields similar performance gains when paired with JES and TS, just as with UCB, whereas with logEI (an EI-family variant) the improvement is negligible. We believe this difference stems from the inherently greedy nature of EI-style methods: since they already concentrate heavily on nearby “promising” regions, multiplying by credit weights does not materially change their ranking or introduce substantial additional guidance. In contrast, acquisitions with explicit or implicit exploration (e.g., UCB, TS, JES) benefit more from the extra bias toward high-contribution regions while retaining their ability to explore.

These findings further confirm the generality of our credit-weighting approach across a variety of acquisition strategies, with performance improving or at least not degrading in most cases.

\begin{figure}[!t]
    \centering
    \begin{minipage}[b]{\textwidth}
        \centering
        \includegraphics[width=1\textwidth]{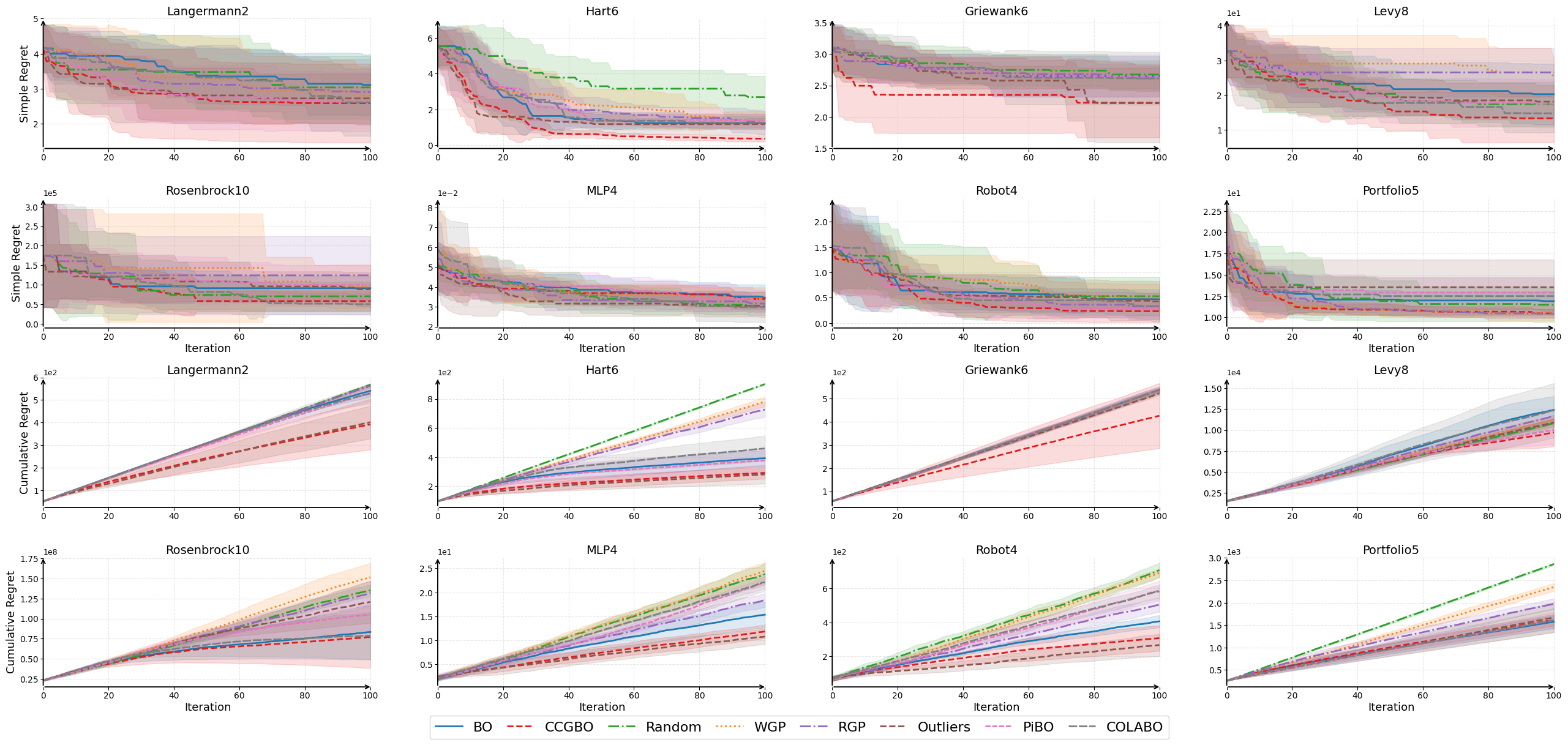}
    \end{minipage}
    \caption{Cumulative regret and Simple regret versus iteration for 8 benchmark functions on TS. Each subplot plots cumulative regret over iterations $t$, comparing the eight baselines. }
    \label{fig:Result_TS}
\end{figure}

\begin{figure}[!t]
    \centering
    \begin{minipage}[b]{\textwidth}
        \centering
        \includegraphics[width=1\textwidth]{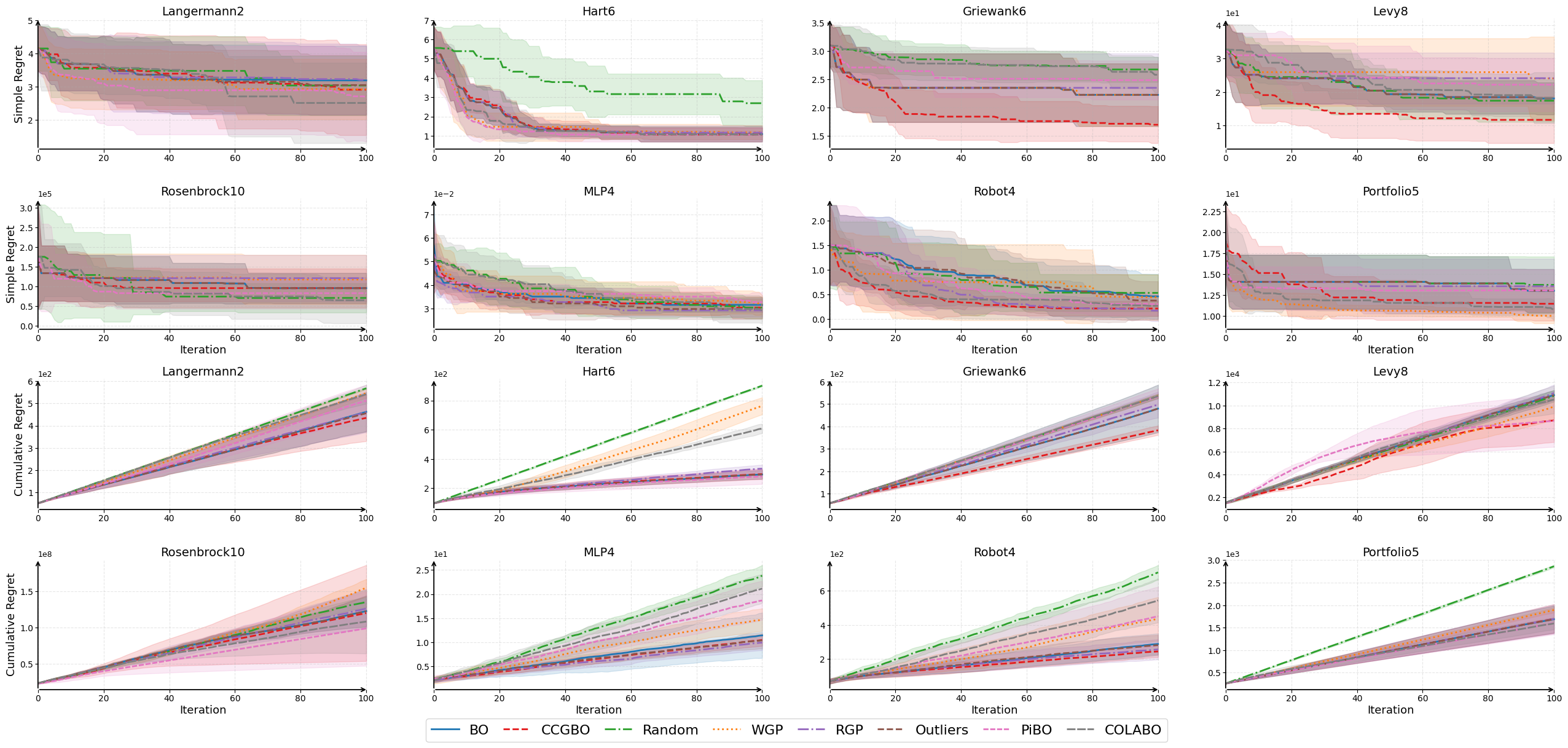}
    \end{minipage}
    \caption{Cumulative regret and Simple regret versus iteration for 8 benchmark functions on LogEI. Each subplot plots cumulative regret over iterations $t$, comparing the eight baselines. }
    \label{fig:Result_EI}
\end{figure}

\begin{figure}[!t]
    \centering
    \begin{minipage}[b]{\textwidth}
        \centering
        \includegraphics[width=1\textwidth]{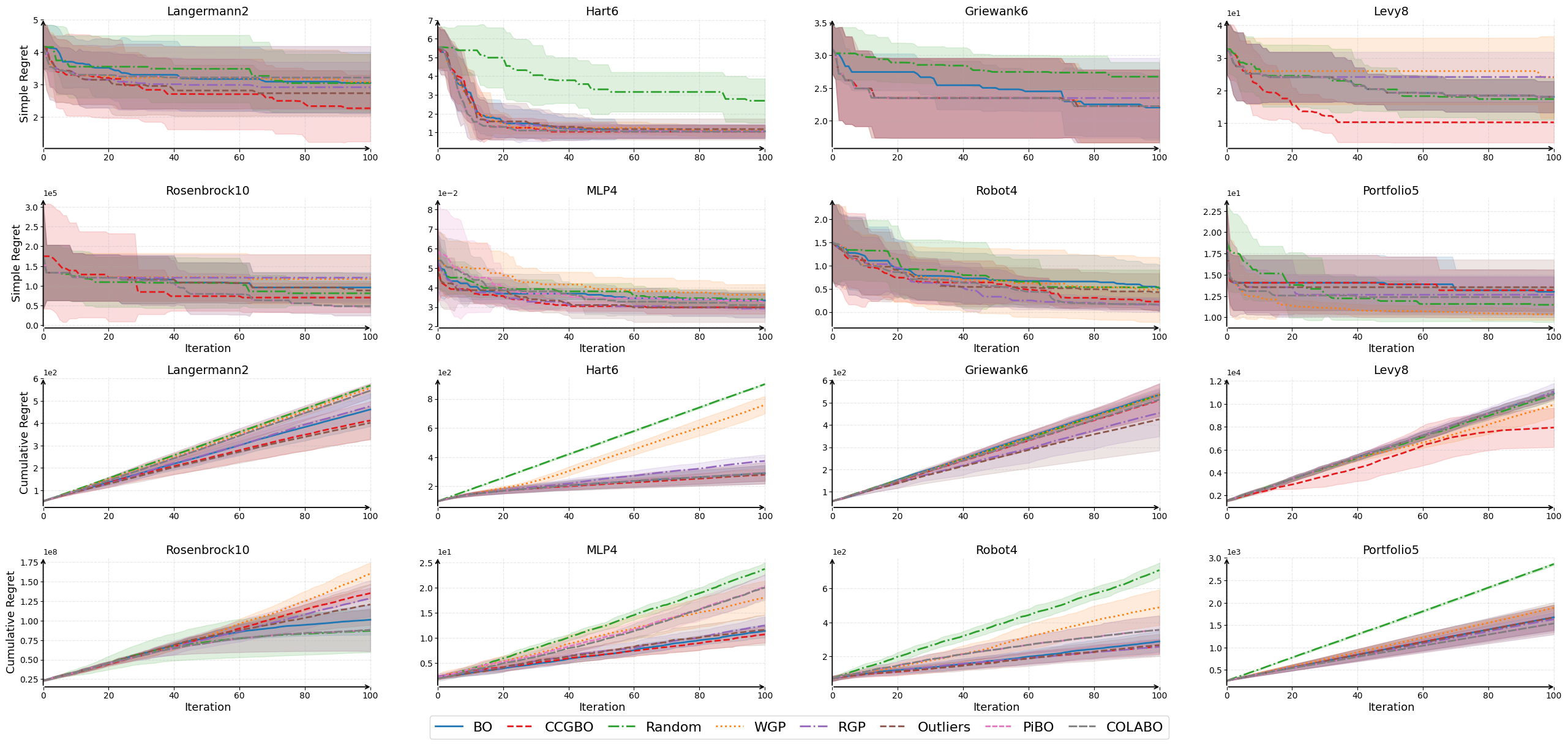}
    \end{minipage}
    \caption{Cumulative regret and Simple regret versus iteration for 8 benchmark functions on JES. Each subplot plots cumulative regret over iterations $t$, comparing the eight baselines. }
    \label{fig:Result_JES}
\end{figure}

\section{High-dimensional Tasks}

Tackling the intrinsic challenges of very high‑dimensional BO is not the primary focus or contribution of our paper, but rather a distinct research topic in its own right. To nonetheless demonstrate CCGBO’s behavior, we tested all methods on high‑dimensional benchmarks. 

However, directly comparing against algorithms specifically designed for high‑D BO would be unfair, since those methods adapt their kernels or acquisition functions to that setting. Instead, we adopted the straightforward “Vanilla Bayesian Optimization Performs Great in High Dimensions” \citep{hvarfner2024vanilla} strategy: in a standard GP + acquisition framework, one simply shifts the log‐lengthscale prior upward by $\tfrac12\log D$ (i.e.\ scales the lengthscale by $\sqrt{D}$) and fixes the signal variance to 1. We applied this modified GP backbone across all CCGBO and baseline algorithms to ensure a fair comparison. 

In the experiments, we follow \citep{hvarfner2024vanilla} to employ Embedded Test Functions to evaluate the performance of Bayesian optimization algorithms in high-dimensional spaces. These functions are constructed by embedding classical low-dimensional benchmark functions into higher-dimensional spaces.
Given a $d_{eff}$-dimensional benchmark function $f: \mathbb{R}^{d_{eff}} \rightarrow \mathbb{R}$, we embed it into a $D$-dimensional space ($D \gg d_{eff}$) to construct a new objective function $f_{embed}: \mathbb{R}^{D} \rightarrow \mathbb{R}$:
\(
f_{embed}(\mathbf{x}) = f(\mathbf{x}_{1:d_{eff}})
\)
where $\mathbf{x}_{1:d_{eff}}$ denotes the first $d_{eff}$ components of the input vector $\mathbf{x} \in \mathbb{R}^{D}$. This means that the objective function value depends only on the first $d_{eff}$ dimensions (effective dimensions), while the remaining $D - d_{eff}$ dimensions have no effect on the function value.
In this work, two classical benchmark functions are adopted as bases:
(1) Levy function: $4$ effective dimensions, characterized by multiple local optima, used to evaluate global search capability.
(2) Hartmann function: $6$ effective dimensions, with a complex non-convex landscape, widely used in optimization benchmarking.
These functions are embedded into $25$, $100$, $300$, and $1000$-dimensional spaces for experiments. For example, \texttt{Levy4\_1000} denotes embedding the $4$-dimensional Levy function into a $1000$-dimensional space.

The resulting performance on high‑dimensional tasks is reported in Figure~\ref{fig:High}.
When these algorithms all get some performance improvements with the help of the algorithm proposed in "Vanilla Bayesian Optimization Performs Great in High Dimensions", they no longer struggle in high-dimensional tasks like standard GP. In such a comparison, our CCGBO still achieves the good results.

\begin{figure}[!t]
    \centering
    \begin{minipage}[b]{\textwidth}
        \centering
        \includegraphics[width=1\textwidth]{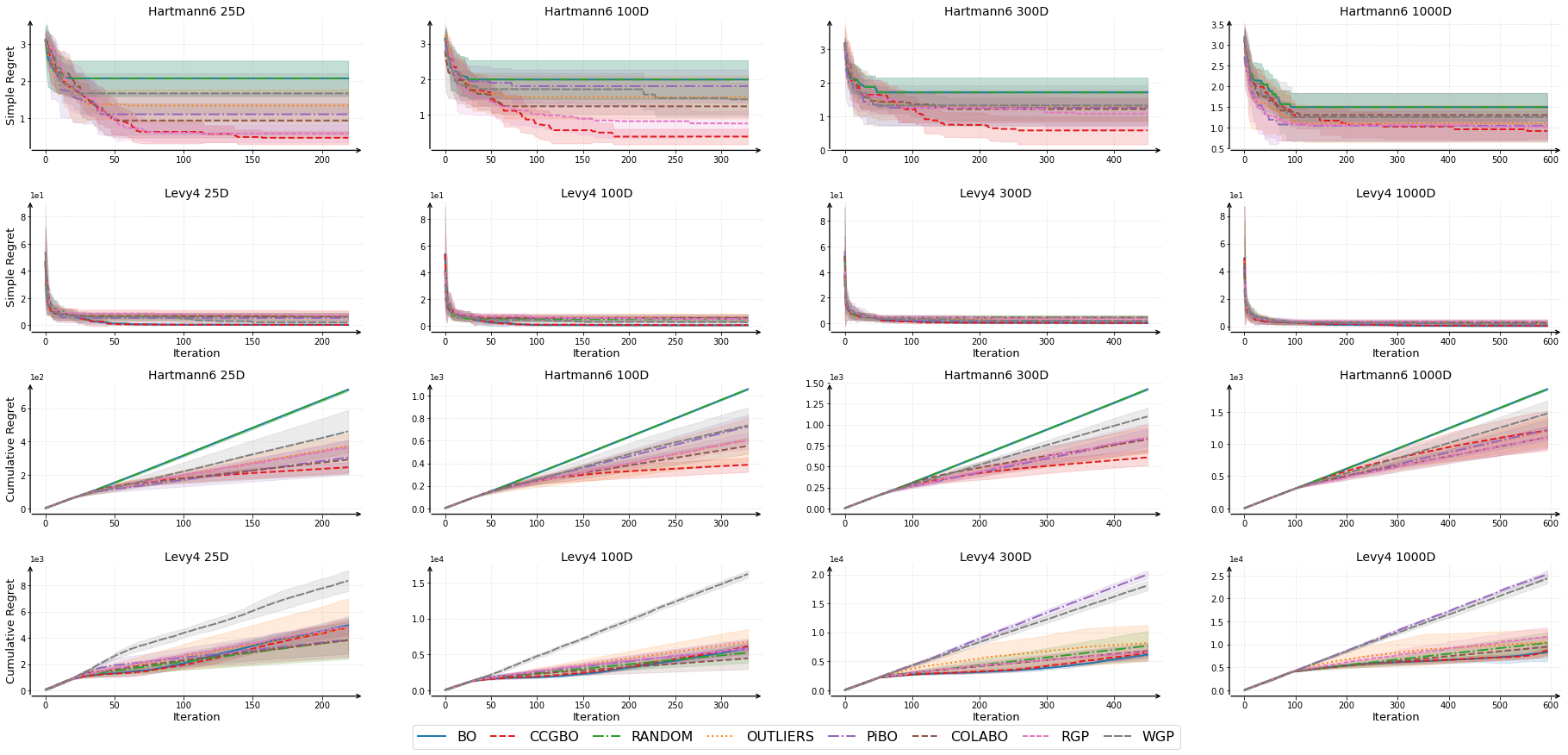}
    \end{minipage}
    \caption{Simple regret versus iteration for eight high-dimensional test problems. Each subplot shows the mean simple regret over iterations $t$ for CCGBO and baselines. }
    \label{fig:High}
\end{figure}

\section{Ablation Study on ``half‐life'' parameter M}

\begin{figure}[!t]
    \centering
    \begin{minipage}[b]{\textwidth}
        \centering
        \includegraphics[width=0.6\textwidth]{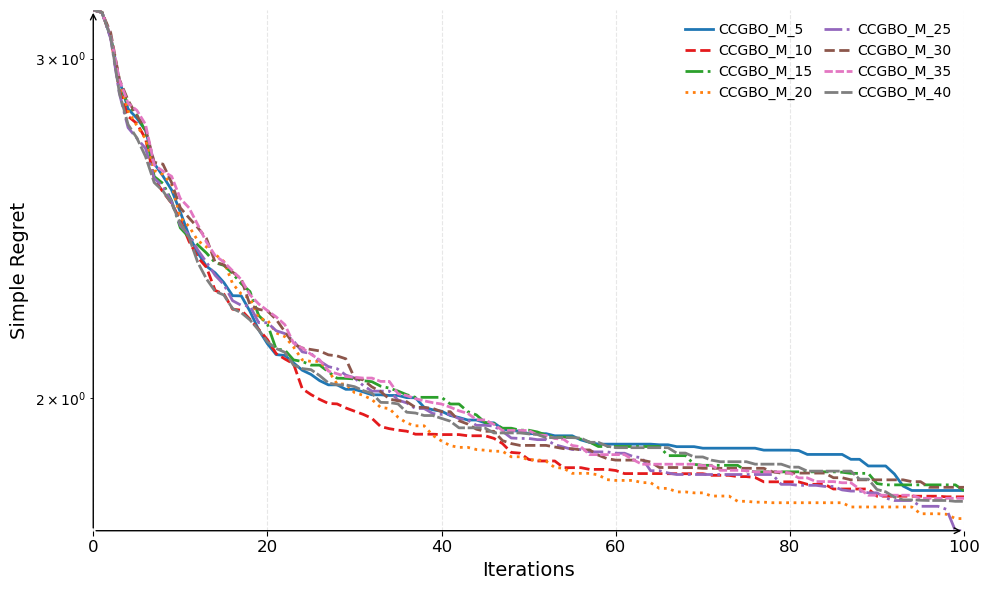}
    \end{minipage}
    \caption{%
    {\bf Ablation of ``half‐life'' parameter \(M\) on simple regret.}  
    Curves show simple regret over iterations for CCGBO with \(M\in\{5,10,15,20,25,30,35,40\}\).  
    }
    \label{fig:ablation_M}
\end{figure}

Figure~\ref{fig:ablation_M} shows how varying the credit decay constant \(M\) affects simple regret.  Recall that in our credit‐weighted UCB, the per‐iteration exponent on the normalized credit \(\pi(\mathbf{x})\) is \(\tau / (1 + n/M)\), so \(M\) controls how quickly the extra credit weighting fades:
(1)Small \(M\) (e.g.\ \(M=10\)).  The exponent \(\tau/(1 + n/M)\) drops rapidly as \(n\) increases, causing \(w_n(\mathbf{x})\to1\) after only a few iterations.  Consequently, the acquisition function reverts almost immediately to vanilla UCB, and the benefit of credit guidance is lost, slowing late‐stage convergence.
(2))Large \(M\) (e.g.\ \(M\ge30\)).  The credit exponent decays very slowly, so credit influence remains strong throughout most of the run.  This over‐emphasis on early high‐credit regions leads to excessive exploitation and a premature plateau in the simple regret curve.
(3))Moderate \(M\) (e.g.\ \(15 \le M \le 25\)).  The exponent decays at a controlled rate: early iterations are properly biased by reliable credit signals, while later iterations smoothly return to a balanced exploration–exploitation regime.  In our experiments, \(M=20\) yielded the lowest final simple regret and the most consistent downward trend.

\section{Ablation Study on MC surrogate}

\begin{figure}[!t]
    \centering
    \begin{minipage}[b]{\textwidth}
        \centering
        \includegraphics[width=0.6\textwidth]{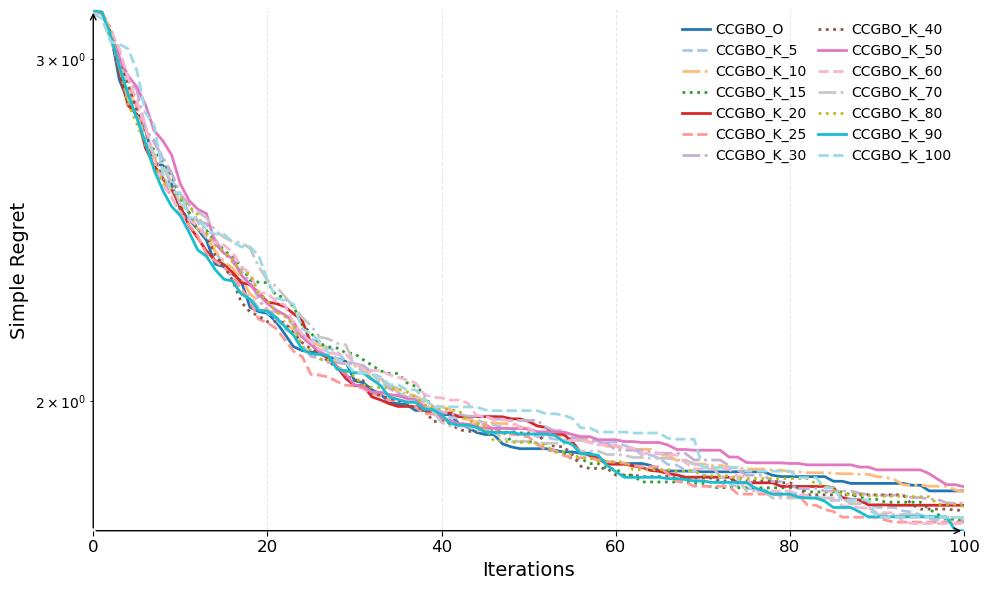}
    \end{minipage}
    \caption{%
    {\bf Ablation of MC surrogate size \(K\) on simple regret.}  
    Each curve plots the mean simple regret over iterations for CCGBO using different values of the credit decay parameter \(M\in\{5,10,15,20,25,30,35,40\}\).
    }
    \label{fig:ablation_K}
\end{figure}

Figure~\ref{fig:ablation_K} illustrates how the number of GP posterior samples \(K\) used in the Monte Carlo surrogate affects optimization performance:
(1))Small \(K\) (e.g.\ \(K\le10\)). The proxy \(Z_t\) is estimated from very few samples, making the resulting credit scores noisy. Early iterations therefore receive unstable guidance, and the algorithm behaves similarly to standard UCB with erratic credit weighting, yielding modest gains over the baseline.
(2)Moderate \(K\) (e.g.\ \(15 \le K \le 30\)). The proxy becomes sufficiently accurate to produce reliable credit signals, while still incurring moderate computation. In this regime, CCGBO attains the fastest and most consistent reduction in simple regret, with \(K\approx20 - 30\) achieving the lowest final regret.
(3))Large \(K\) (e.g.\ \(K\ge60\)). Further increasing \(K\) yields diminishing returns: the proxy is more precise, but the extra computational cost grows. These results demonstrate that a moderate MC surrogate size (around \(K=20 - 30\)) strikes the best balance between proxy fidelity and computational efficiency, enabling CCGBO to leverage accurate counterfactual credits for rapid convergence without over‐exploitation or excessive cost.

\section{Ablation Study on Nearest Neighbors H}

\begin{figure}[!t]
    \centering
    \begin{minipage}[b]{\textwidth}
        \centering
        \includegraphics[width=0.6\textwidth]{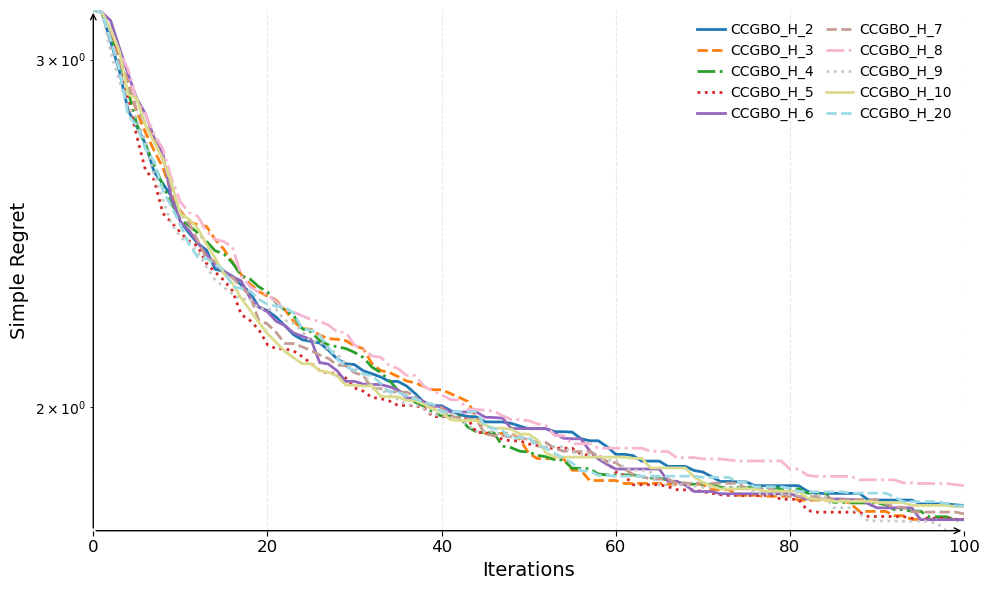}
    \end{minipage}
    \caption{
    \textbf{Ablation of KNN neighbor parameter \(H\) on simple regret.}  
    Each curve shows the mean simple regret over 50 runs for CCGBO using different numbers of nearest neighbors \(H\in\{2,3,\dots,10,20\}\) to propagate discrete credits to continuous candidate points.%
    }
    \label{fig:ablation_H}
\end{figure}

Figure~\ref{fig:ablation_H} evaluates how the number of neighbors \(H\) used in the KNN‐based credit propagation affects optimization performance:
(1)Small \(H\) (e.g.\ \(H=2-3\)). Credit propagation is highly localized: each candidate inherits credit from only a few nearest samples, resulting in a very rugged credit field. Early iterations show erratic behavior and slower convergence due to noisy weight estimates.
(2)Moderate \(H\) (e.g.\ \(H=4-7\))). The credit field balances locality and smoothness. Candidates in high‐credit regions receive coherent weighting, while edge points still retain exploratory potential. This setting yields the fastest decline in simple regret and the lowest final regret.
(3)Large \(H\) (e.g.\ \(H\ge10\)). Credit smoothing becomes excessive: distant, low‐credit samples unduly influence each candidate. The credit contrast diminishes, causing the acquisition to revert toward uniform UCB behavior and slowing late‐stage improvement. These results indicate that a moderate neighbor count (\(H\approx5-7\)) provides the best trade‐off between a stable yet discriminative credit field, enabling CCGBO to effectively leverage historical contributions without over‐smoothing.

\section{Ablation Study on Additional Hyperparameters}

In this section, we provide ablation studies on the remaining hyperparameters: the credit influence parameter \(\lambda\), the UCB exploration factor \(\beta_t\), the numerical stability constant \(\varepsilon_c\), the minimum credit bound \(r_{\min}\), and the credit decay exponent \(\tau\).

\paragraph{Credit Influence Parameter \(\lambda\).}
The parameter \(\lambda \in [0,1]\) controls the strength of credit influence in the acquisition function. When \(\lambda = 0\), CCGBO reduces to standard UCB; when \(\lambda = 1\), the credit weighting is fully applied. Table~\ref{tab:ablation_lambda} reports the final simple regret on the Hartmann6 benchmark for different values of \(\lambda\).

\begin{table}[!ht]
\centering
\caption{Ablation of credit influence parameter \(\lambda\) on simple regret (Hartmann6, mean \(\pm\) std over 50 runs).}
\label{tab:ablation_lambda}
\small
\begin{tabular}{lccccccccc}
\toprule
\(\lambda\) & 0.1 & 0.2 & 0.3 & 0.4 & 0.5 & 0.6 & 0.7 & 0.8 & 0.9 \\
\midrule
SimRegret & 1.17{\tiny$\pm$0.55} & 1.17{\tiny$\pm$0.44} & 1.15{\tiny$\pm$0.42} & 1.13{\tiny$\pm$0.41} & 1.03{\tiny$\pm$0.33} & 1.03{\tiny$\pm$0.35} & 1.06{\tiny$\pm$0.36} & 1.06{\tiny$\pm$0.36} & 1.06{\tiny$\pm$0.36} \\
\bottomrule
\end{tabular}
\end{table}

We observe that performance improves as \(\lambda\) increases from small to moderate values and maintains similar performance in the range \(\lambda \ge 0.5\). This indicates that CCGBO is robust to the choice of \(\lambda\) over a wide range, rather than requiring careful tuning.

\paragraph{Comparison with GP-UCB under Varying \(\beta_t\).}
A natural question is whether simply reducing the exploration parameter \(\beta_t\) in standard GP-UCB could achieve comparable performance to CCGBO by inducing more exploitative behavior. To investigate this hypothesis, we conducted a systematic comparison between GP-UCB with varying \(\beta_t\) values and CCGBO with a fixed \(\beta_t = 2.576\). Table~\ref{tab:beta_comparison} reports the final simple regret on three representative benchmark problems.
 
\begin{table}[!ht]
\centering
\caption{Comparison of GP-UCB with varying \(\beta_t\) values against CCGBO on simple regret (mean \(\pm\) std over 50 runs). Lower is better.}
\label{tab:beta_comparison}
\small
\setlength{\tabcolsep}{4pt}
\begin{tabular}{lccccccc}
\toprule
Problem & \(\beta=0.5\) & \(\beta=1.0\) & \(\beta=2.0\) & \(\beta=2.576\) & \(\beta=4.6\) & CCGBO \\
\midrule
Rosenbrock2 & 45.1{\scriptsize$\pm$32.5} & 60.4{\scriptsize$\pm$52.2} & 58.3{\scriptsize$\pm$65.7} & 55.5{\scriptsize$\pm$68.9} & 58.9{\scriptsize$\pm$64.6} & \textbf{42.2}{\scriptsize$\pm$12.5} \\
Hartmann6 & 4.86{\scriptsize$\pm$1.41} & 3.88{\scriptsize$\pm$1.52} & 1.85{\scriptsize$\pm$1.25} & 1.17{\scriptsize$\pm$0.53} & 1.24{\scriptsize$\pm$0.34} & \textbf{1.03}{\scriptsize$\pm$0.33} \\
Levy8 & 22.9{\scriptsize$\pm$5.89} & 21.3{\scriptsize$\pm$5.13} & 19.1{\scriptsize$\pm$4.80} & 17.7{\scriptsize$\pm$4.52} & 20.5{\scriptsize$\pm$4.56} & \textbf{11.0}{\scriptsize$\pm$6.38} \\
\bottomrule
\end{tabular}
\end{table}
 
The results reveal that reducing \(\beta_t\) alone cannot replicate the performance of CCGBO. On the unimodal Rosenbrock2 function, smaller \(\beta_t\) values yield marginal improvements but still underperform CCGBO, which achieves both lower mean regret and substantially reduced variance. On multimodal landscapes (Hartmann6 and Levy8), decreasing \(\beta_t\) suppresses exploration globally, causing GP-UCB to converge prematurely to local optima. In contrast, CCGBO selectively amplifies exploitation only in regions with high estimated contribution while preserving standard exploration--exploitation balance elsewhere. Furthermore, the time-decaying weighting mechanism ensures that CCGBO gradually reverts to standard UCB behavior in later iterations, thereby avoiding entrapment in local optima.

Overall, these ablation studies demonstrate that CCGBO maintains robust performance across a wide range of hyperparameter settings, indicating that the reported improvements are not artifacts of careful tuning but reflect the inherent benefits of the counterfactual credit mechanism.

\clearpage

\end{document}